\def\year{2021}\relax
\documentclass[letterpaper]{article} 
\usepackage{aaai21}  
\usepackage{times}  
\usepackage{helvet} 
\usepackage{courier}  
\usepackage[hyphens]{url}  
\usepackage{graphicx} 
\usepackage{graphicx}  
\usepackage{natbib}  
\usepackage{caption} 
\usepackage{balance}
\usepackage{booktabs}
\usepackage{subfigure}
\usepackage{multirow}
\usepackage{color}
\usepackage{comment}
\usepackage{xcolor}
\usepackage{helvet}
\usepackage{courier}
\usepackage{graphicx}
\usepackage{subfigure}
\usepackage{balance}
\usepackage{amsmath,amsthm,amssymb}
\usepackage{textcomp,booktabs}
\usepackage{footmisc}
\usepackage[normalem]{ulem}
\usepackage{multirow}
\usepackage{setspace}
\usepackage{bigstrut}
\usepackage{array}
\usepackage{xcolor}
\usepackage{amsfonts}
\usepackage{bm}
\usepackage{bbm}
\usepackage{url}
\usepackage{threeparttable}
\usepackage{algorithm}
\usepackage{algorithmic}
\usepackage{wrapfig}
\usepackage{amsmath}
\usepackage{amsfonts,amssymb}
\usepackage[switch]{lineno}  %

\newcommand{\x}{\boldsymbol{x}}

\newcommand{\A}{\mathcal{A}}

\newtheorem{Def}{Definition}
\newtheorem{prop}{Proposition}
\newtheorem{lemma}{Lemma}
\newtheorem{thm}{Theorm}

\title{Bi-Classifier Determinacy Maximization for Unsupervised Domain Adaptation}
\author {Shuang Li,\textsuperscript{\rm 1}
Fangrui Lv,\textsuperscript{\rm 1}
Binhui Xie,\textsuperscript{\rm 1}
Chi Harold Liu,\textsuperscript{\rm 1,\thanks{Corresponding author.}}
Jian Liang,\textsuperscript{\rm 2}
Chen Qin\textsuperscript{\rm 3}\\
}
\affiliations {\textsuperscript{\rm 1}School of Computer Science and Technology, Beijing Institute of Technology, China\\
\textsuperscript{\rm 2}Alibaba Group, China\\
\textsuperscript{\rm 3}	Institute for Digital Communications, School of Engineering, University of Edinburgh, United Kingdom\\
\{shuangli,fangruilv,binhuixie,chiliu\}@bit.edu.cn,
liangjianzb12@gmail.com,
Chen.Qin@ed.ac.uk
}

\begin{document}
\maketitle

\begin{abstract}   
Unsupervised domain adaptation challenges the problem of transferring knowledge from a well-labelled source domain to an unlabelled target domain. Recently, adversarial learning with bi-classifier has been proven effective in pushing cross-domain distributions close. Prior approaches typically leverage the disagreement between bi-classifier to learn transferable representations, however, they often neglect the classifier determinacy in the target domain, which could result in a lack of feature discriminability. In this paper, we present a simple yet effective method, namely \textit{Bi-Classifier Determinacy Maximization} (BCDM), to tackle this problem. Motivated by the observation that target samples cannot always be separated distinctly by the decision boundary, here in the proposed BCDM, we design a novel classifier determinacy disparity (CDD) metric, which formulates classifier discrepancy as the class relevance of distinct target predictions and implicitly introduces constraint on the target feature discriminability. To this end, the BCDM can generate discriminative representations by encouraging target predictive outputs to be consistent and determined, meanwhile, preserve the diversity of predictions in an adversarial manner. Furthermore, the properties of CDD as well as the theoretical guarantees of BCDM's generalization bound are both elaborated. Extensive experiments show that BCDM compares favorably against the existing state-of-the-art domain adaptation methods. 
\end{abstract}
    
\section{Introduction}

The significant progress on computer vision tasks has been dominated by deep neural networks (DNNs) trained on large-scale datasets, such as image classification~\cite{alexnet}, semantic segmentation~\cite{chen2018deeplab} and object detection~\cite{ren2015faster}. Despite their tremendous success, collecting sufficient labelled data in real-world applications is always onerous and extremely costly. In this regard, it is highly desirable to develop unsupervised domain adaptation (UDA) techniques~\cite{survey,TCA,GFK} to address the domain shift problem existing between unlabelled target domain and labelled source domain.

Latest advances of deep UDA methods, leveraging the extraordinary feature extraction ability of DNNs, have been continuously pushing forward the boundaries of UDA~\cite{BNM,MCC,BSP,AFN,DCAN}. To mitigate the domain shift, some UDA approaches aim at minimizing the well-defined distance metrics between domains~\cite{DDC,DAN}. Works along this line are based on the theoretical analysis in~\cite{A-distance} that minimizes the divergence across source and target domains. On par with them, fruitful line of works utilizing generative adversarial nets (GANs)~\cite{GAN} have also achieved remarkable performance by learning domain-invariant representations adversarially in a two player min-max game~\cite{DANN,MCD}.

\begin{figure*}[!htbp]
  \centering
  \includegraphics[width=0.98\textwidth]{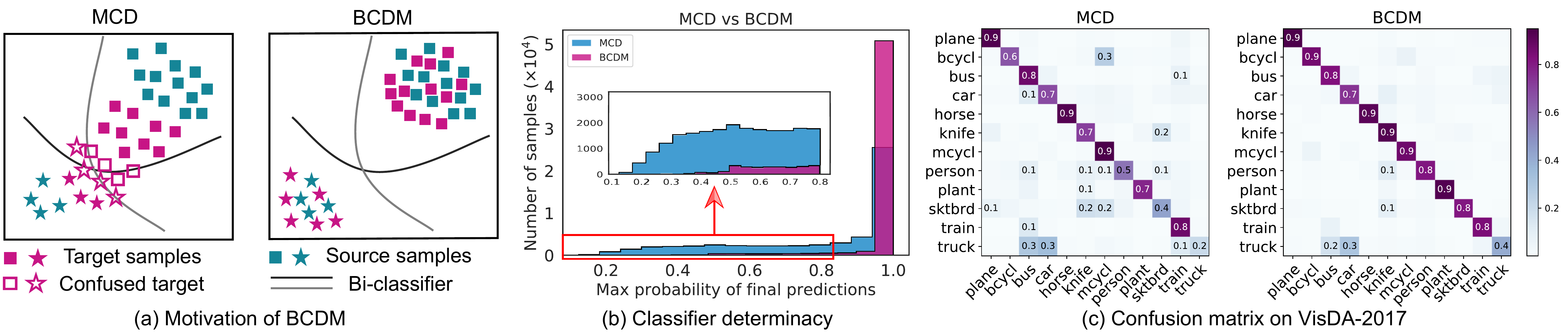}
  \caption{(a): Comparison of previous and the proposed bi-classifier matching methods. (b)(c): Empirically unveil that \textit{higher classifier determinacy} indicates more discriminative representations on VisDA-2017 dataset. We first show the frequency of max probability of each target sample's prediction, where the MCD obtains more uncertain predictions with small values while our BCDM obtains more confident predictions ($>0.9$). Similarly, the tendency that bi-classifier in MCD gives ambiguous predictions also results in fewer target samples on the diagonal in (c). It implicitly implies that scarcity of determinacy during bi-classifier adversarial training deteriorates the discriminability in the target domain.}
  \label{Fig_motivation}
\end{figure*}

There exist two popular paradigms to conduct adversarial domain adaptation either by constructing a domain discriminator or by utilizing two distinct classifiers. 
However, the first category will bring side effects of the deterioration of feature discriminability, though the feature transferability is strengthened~\cite{BSP}. As for the second paradigm, the selection of classifier discrepancy loss between two task-specific classifiers is critical for expected adaptability. 
The straightforward $\ell_1$ distance between predictions $\boldsymbol{p}_1$ and $\boldsymbol{p}_2$, i.e., $|\boldsymbol{p}_1-\boldsymbol{p}_2|_1$, in Maximum Classifier Discrepancy~\cite{MCD}, merely measures the bi-classifier discrepancy at the corresponding class positions while ignoring the constraint of the relevance information between classes.
Thus the certainty of predictions cannot be guaranteed, which may further hurt the learning behavior. For instance, there could be uncertain predictions such as values of $[0.34,0.33,0.33]$ and $[0.34,0.33,0.33]$ by minimizing the $\ell_1$ distance between bi-classifier, which results in generating confused target features.
Such situation can lead to ambiguous decisions and thus should be avoided. As illustrated in Fig.~\ref{Fig_motivation}, due to the domain shift dilemma, numerous target samples are prone to be around the decision boundary.  In other words, the classifier determinacy that can put confident labels on target samples, is lacking, thus impacting the feature discriminability. 

To tackle the aforementioned issues, we propose a Bi-Classifier Determinacy Maximization (BCDM) method together with a novel metric called classifier determinacy disparity (CDD). Specifically, CDD contains all the probabilities that the predictions of bi-classifier are inconsistent, which is defined as the summation of product between bi-classifier different class predictions. Regarding the training procedure, we first maximize the CDD loss which enforces bi-classifier to generate disagreement across classes, so as to explore diverse output spaces under the source supervision. And then we minimize the CDD loss to push the two domains together class-wisely by generating discriminative features that make the predictions consistent and determined. In this way, our method is able to delve into pushing the target samples away from the decision boundary. Consequently, the prediction consistency guarantees target samples to be close to the source regions class-wisely. The bi-classifier determinacy equips target features with categorical discriminability, which further improves the accuracy.

On top of the proposed CDD metric, the BCDM is able to simultaneously enhance the classifier determinacy and prediction diversity by adversarially optimizing the CDD loss. 
Our method has been shown to be simple yet effective.
In addition, theoretical guarantees about BCDM in the aspect of generalization error bound on the target domain are elaborated. 
The proposed method is validated on various applications and scenarios including a toy dataset, four cross-domain image classification benchmark datasets, and a prevalent ``synthetic-2-real'' semantic segmentation set-up. Extensive experiments show that our method achieves significant improvements over state-of-the-art UDA methods.

\section{Related Work}
Existing deep domain adaptation methods can be mainly grouped into two categories: discrepancy optimization based methods~\cite{DAN,DCAN,DRCN} and adversarial learning based methods~\cite{DANN,JADA,SDTADT}. In the first group, DAN~\cite{DAN} has proposed to reduce the gap between domains via minimizing a discrepancy metric named maximum mean discrepancy (MMD)~\cite{MMD}. Later, JAN~\cite{JAN} utilizes modified versions of MMD to align joint distributions for more effective transfer.

Adversarial learning based UDA methods represent another line of works, which share the similar spirits with GANs~\cite{GAN}.  They mainly focus on learning indistinguishable representations for data in both source and target domains in an adversarial manner. Specifically, one type of approaches exploits a domain discriminator to distinguish whether the input samples are from source or target domain, while a generator is trained to generate features that can confuse the discriminator~\cite{DANN,MADA,AdaSegNet}. The other group of methods can utilize two distinct task-specific classifiers acting as a domain discriminator~\cite{ADR,MCD,SWD,CLAN,STAR}. 

Our proposed method falls into the adversarial learning category and inherits the similar bi-classifier training strategy. To name a few, MCD~\cite{MCD} uses the $\ell_1$ distance to calculate the classifier discrepancy which only considers the disagreement between the two classifiers with respect to the same class prediction. By comparison, the sliced Wasserstein distance in~\cite{SWD} is more powerful to capture the geometrically meaningful dissimilarity between two classifiers through optimal transport, however, this method not only increases the computation complexity but requires a sophisticated hyper-parameter tuning process. Very recently, STAR finds that more classifiers can achieve better adaptation performance and models infinite number of classifiers through reparameterization trick~\cite{STAR}. CLAN~\cite{CLAN} applies different adversarial weight to different pixels in semantic segmentation without explicitly incorporating the classes into the model. 

Though these approaches have yielded promising results, they actually neglect the determinacy of the classifier, which may result in the decision boundary passing through high data density and misalignment among classes. To address this, here we propose a bi-classifier determinacy maximization (BCDM) algorithm with a novel metric, which formulates the classifier discrepancy as the relevance of different class predictions with theoretical justifications. Extensive experiments have demonstrated that our method can significantly outperform the competing baseline approaches.

\section{Bi-Classifier Determinacy Maximization}
Assume a labelled source domain $\mathcal{S}=\{\mathcal{X}_s,\mathcal{Y}_s\}=\{(\x^s_i, y^s_i)\}^{n_s}_{i=1}$ with $n_s$ examples, where $y^s_i$ is the corresponding label of $\x^s_i$, and an unlabelled target domain $\mathcal{T}=\{\mathcal{X}_t\}=\{\x^t_j\}^{n_t}_{j=1}$ with $n_t$ samples, the goal of our work is to learn a deep neural network $\phi : \mathcal{X}_t \rightarrow \mathcal{Y}_t$ that can learn discriminative representations with determined predictions on the target domain. The model consists of a feature generator $G$ and two distinct classifiers $C_1$ and $C_2$, and $\theta_g, \theta_{c_1}$ and $\theta_{c_2}$ are the corresponding network parameters. To better understand, we show the overview of the BCDM in Fig. \ref{Fig_architecture}. 

In this section, we first introduce a novel classifier determinacy disparity (CDD) to measure the classifier discrepancy. Then, we report our bi-classifier determinacy maximization (BCDM) and investigate its theoretical guarantees.

\subsection{Classifier Determinacy Disparity}
\label{Section3.1}
To measure the classifier discrepancy, we first propose a classifier determinacy disparity (CDD) metric. Given $\boldsymbol{p}_1$ and $\boldsymbol{p}_2$ as the two classifiers' softmax probabilities, it can be seen that $\boldsymbol{p}_1$ and $\boldsymbol{p}_2 \in \mathbb{R}^{ K \times 1}$ satisfies:
\begin{small}
\begin{equation}\label{output predictions}
\sum^K_{k=1}{p}_j^k = 1, ~~~~\mathrm{s.t.}~ {p}_j^k \ge 0,~~ \forall k = 1,\cdots, K, ~~~j = 1,2,
\end{equation}
\end{small}%
where $p_j^k$ denotes the $k$-th element of the softmax probabilistic output from $C_j$, i.e., the probability that $C_j$ classifies the sample into the $k$-th class, and $K$ is the amount of possible categories. Previous approach, i.e., MCD, deploys $\ell_1$ distance $|\boldsymbol{p}_1-\boldsymbol{p}_2|_1$ to measure the discrepancy between two prediction distributions. However, $\ell_1$ distance only considers the similarity between $\boldsymbol{p}_1$ and $\boldsymbol{p}_2$ at the same position of the corresponding class, and thus ignores the relevance between classes, which may affect the confidence of predictions. For instance, though the $\ell_1$ distance between $\boldsymbol{p}_1=[0.34,0.33,0.33]$ and $\boldsymbol{p}_2=[0.34, 0.33, 0.33]$ is small to zero, such predictions are rather uncertain across classes and thus the classifiers are less robust to perturbations. Therefore, we investigate the classifier discrepancy by Bi-classifier Prediction Relevance Matrix $\mathbf{A}$:
\begin{small}
\begin{equation}\label{output predictions}
\mathbf{A} = \boldsymbol{p}_1\boldsymbol{p}^\top_2,
\end{equation}
\end{small}%
where $\mathbf{A}$ is a square matrix of size $K \times K$. $A_{mn}=p_1^m p_2^n (m,n \in \{1,2,\cdots,K\})$ is the element in the $m$-th row and $n$-th column of $\mathbf{A}$, which represents the probability product of $C_1$ classifying the sample into the $m$-th category and $C_2$ classifying it into the $n$-th category. In other words, $\mathbf{A}$ could effectively evaluate the bi-classifier prediction relevance across different classes.With the characteristics of $\mathbf{A}$,  to minimize the classifier discrepancy with respect to the prediction relevance, we have no alternative but make the two predictions consistent and correlated, i.e., maximizing the diagonal elements of $\mathbf{A}$. This also enables the predictions to be determined with high confidence for the predicted class. At the same time, the non-diagonal elements of $\mathbf{A}$ can be seen as fine-grained confusion information of the two classifiers. Therefore, we define the CDD loss as
\begin{small}
\begin{equation}\label{cdd_formalization}
\Gamma(\boldsymbol{p}_1,\boldsymbol{p}_2)= \sum_{m,n=1}^KA_{mn}-\sum_{m=1}^KA_{mm}=\sum_{m\neq n}^KA_{mn}.
\end{equation}
\end{small}%
\begin{figure}[!htbp]
    \centering
    \includegraphics[width=0.48\textwidth]{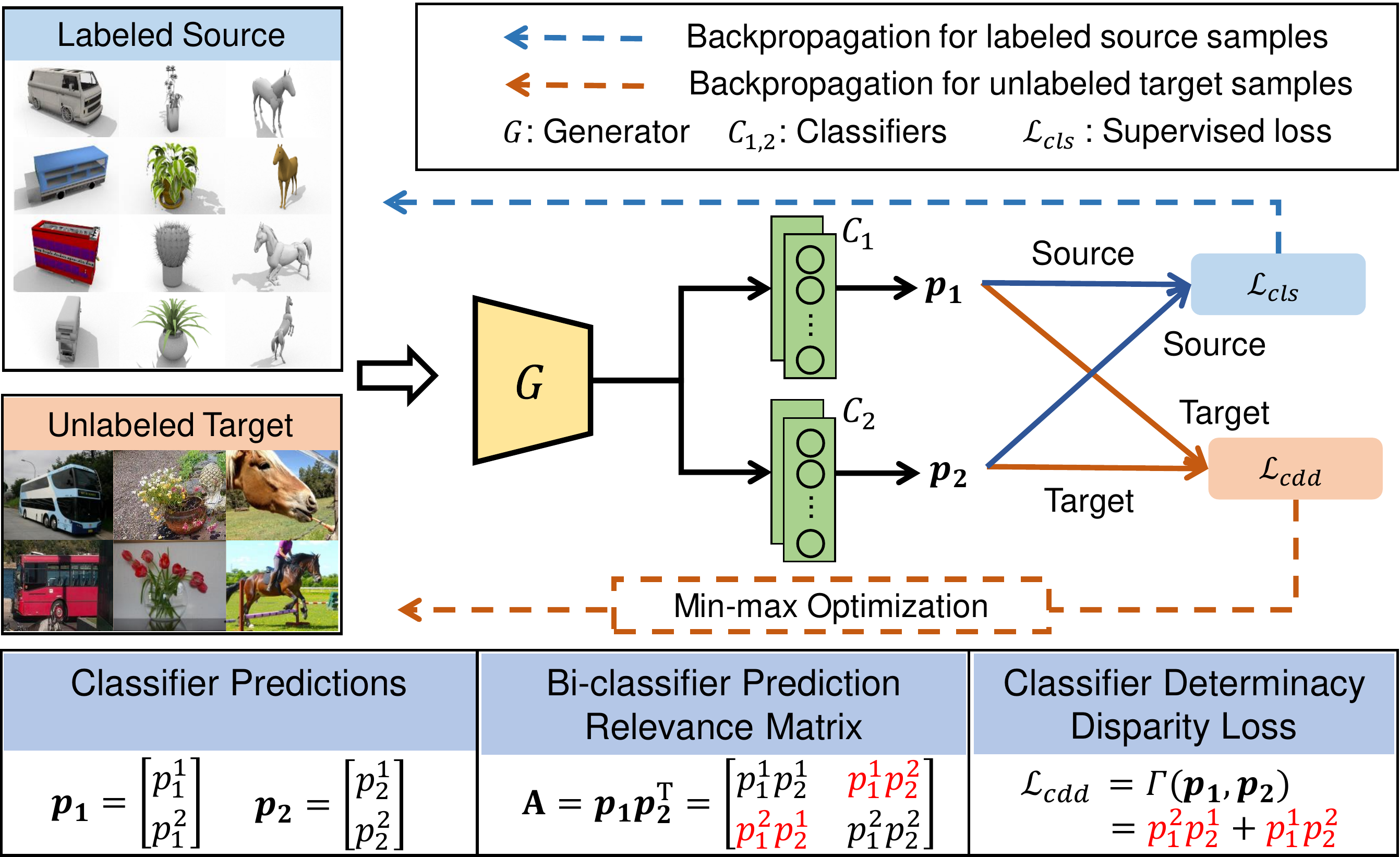}
    \caption{Approach overview. In BCDM, \textbf{step 1} optimizes all the networks by minimizing the source classification loss. Under the source supervision, \textbf{step 2} maximizes the proposed classifier determinacy disparity (CDD) loss on target domain by updating bi-classifier parameters. The generator seeks to generate transferable and discriminative features by minimizing CDD loss on target domain in \textbf{step 3}. An example of calculating CDD loss is also provided in the bottom, where $p_i^j$ denotes the probability of classifier $C_i$ classifying the input sample into the j-th category.}
    \label{Fig_architecture}
\end{figure}%
The first item in Eq. \eqref{cdd_formalization} equals to 1 since $\boldsymbol{p}_1$ and $\boldsymbol{p}_2$ are softmax outputs. We can see that CDD contains all the probabilities that the prediction of $C_1$ is inconsistent with the prediction of $C_2$, so it can be used as a measure to evaluate the discrepancy between predictions. Notably, $\Gamma(\boldsymbol{p}_1,\boldsymbol{p}_2)$ satisfies the properties of a metric and the proof is provided in the supplementary materials. Moreover, only when the two probabilities are consistent and fully determined, such as [1,0,0] and [1,0,0], can the CDD reach the minimum value of zero, which guarantees the discriminability of features.

\subsection{Bi-Classifier Determinacy Maximization}
Based on the introduced CDD metric, here we propose a bi-classifier determinacy maximization (BCDM) algorithm for adversarial domain adaptation. 
In UDA problems, it is a prerequisite to ensure that the classifiers can correctly classify source data. Thus, to fully exploit the supervision from source domain, we first train the whole network to minimize the standard supervised loss on source data as follows:
\begin{small}
\begin{equation}\label{min-sup-loss}
\min_{\theta_g,\theta_{c_1},\theta_{c_2}}\frac{1}{n_s}\sum^{n_s}_{i=1} \mathcal{L}_{cls}(\boldsymbol{x}_i^s,y_i^s)=\frac{1}{2n_s}\sum^{n_s}_{i=1}\sum^{2}_{j=1} \mathcal{L}_{ce}({\boldsymbol{p}_j}_i,y_i^s),
\end{equation}
\end{small}%
where ${\boldsymbol{p}_j}_i = C_j(G(\boldsymbol{x}_i^s))$ and $\mathcal{L}_{ce}(\cdot,\cdot)$ is the cross entropy loss. With the full supervision on source data, the feature discriminability for source data can be well preserved. However, it is worth noting that the learned decision boundary on the source domain cannot be directly transferred to the target domain due to the distribution divergence between the two domains. Therefore, we further propose to train the two distinct classifiers $C_1$ and $C_2$ on target domain in an adversarial manner, which aims to learn transferable features and discriminative decision boundary on target domain. 
To achieve this, here, we utilize the designed metric CDD to measure the classifier discrepancy. 
Specifically, the two classifiers are trained to maximize CDD on target domain data,
\begin{small}
\begin{equation}\label{max-cd-loss}
\max_{\theta_{c_1},\theta_{c_2}}\frac{1}{n_t}\sum^{n_t}_{i=1}\mathcal{L}_{cdd}(\boldsymbol{x}_i^t)
=\frac{1}{n_t}\sum^{n_t}_{i=1}\Gamma({\boldsymbol{p}_1}_i,{\boldsymbol{p}_2}_i)\,.
\end{equation}
\end{small}%
Here, ${\boldsymbol{p}_1}_i$ and ${\boldsymbol{p}_2}_i$ are the softmax outputs of $C_1$ and $C_2$ for target sample $\boldsymbol{x}_i^t$, respectively. Training the classifiers via maximizing Eq.~\eqref{max-cd-loss} can effectively detect target samples that are far from the support of the source domain. In essence, the CDD loss encourages two classifiers to perform differently across categories rather than that of the same category in MCD~\cite{MCD}. Besides, Eq.~(\ref{max-cd-loss}) could equip classifiers with more diverse probabilities in a quite simple way that potentially facilitates the discovery of more ambiguous target data near the decision boundary. 

 Through maximizing Eq. (\ref{max-cd-loss}), target samples are mostly located near the decision boundary so as to be detected with diverse predictions, which brings uncertainty to target feature learning. To encourage yielding representations that are discriminative and achieve bi-classifier determinacy, we further train the generator $G$ on target domain to minimize CDD while fixing the classifiers,
 \begin{small}
 \begin{equation}\label{min-cd-loss}
 \min_{\theta_{g}}\frac{1}{n_t}\sum^{n_t}_{i=1}\mathcal{L}_{cdd}(\boldsymbol{x}_i^t)
 =\frac{1}{n_t}\sum^{n_t}_{i=1}\Gamma({\boldsymbol{p}_1}_i,{\boldsymbol{p}_2}_i).
 \end{equation}
 \end{small}
 
Considering the specific property of CDD that it can reach the minimum value only when the two predictions are consistent and fully determined, i.e., 100\% certainty, the generator $G$ is capable of generating target representations that possess great discriminability by minimizing Eq. (\ref{min-cd-loss}), and thus
further benefits the learning task. In contrast, MCD~\cite{MCD} and SWD~\cite{SWD} cannot guarantee whether the decision boundary separates the categorical clusters in the target domain since they only focus on the consistency of two predictions. 
Obviously, the worst case is that the generator is still able to generate confused target features even if the output probabilities are consistent, which severely hinders the feature discriminability.

\textbf{Remarks.} \textbf{On one hand}, to maximize the CDD, we need to minimize the summation of the diagonal elements of matrix $\mathbf{A}$, i.e., $\boldsymbol{p}_1$ and $\boldsymbol{p}_2$ need to be as inconsistent as possible with totally different categorical predictions. It enables us to explore more possibilities of classification for each target data, which increases prediction diversity.
\textbf{On the other hand}, to minimize $\Gamma(\boldsymbol{p}_1,\boldsymbol{p}_2)$ on the target set, the bi-classifier predictions on target data need to be highly correlated and with high certainty, i.e., ideally  $\boldsymbol{p}_1=\boldsymbol{p}_2$ and with probability 1 for the predicted class.
Via this minimization, the predictions should be fully determined and thus push the target domain distribution far away from the decision boundary, which stimulates $G$ to generate discriminative representations. Through such adversarial learning of the CDD loss between the two prediction outputs, we promote the feature discriminability and preserve the diversity of predictions simultaneously. 
To be clear, we summarize the training process of BCDM based on the above discussions as Algorithm~\ref{alg:CDD}. Deep Embedded Validation~\cite{DEV_2019_ICML} is conducted to select hyper-parameters, and then we fix $\alpha = 0.01$ in all experiments.

\begin{algorithm}[!htbp]
    \caption{The Algorithm of BCDM for UDA.}
    \label{alg:CDD}
    \begin{algorithmic} [1]
    \REQUIRE Source samples $\{(\x^s_i, y^s_i)\}^{n_s}_{i=1}$; Target samples $\{\x^t_j\}^{n_t}_{j=1}$; Trade-off parameter: $\alpha$ Batch size: $B$; max$\_$iteration.
    \ENSURE Optimal parameters $\widehat \theta_g, \widehat \theta_{c_1}, \widehat \theta_{c_2}$ of feature generator $G$ and task-specific classifiers $C_1, C_2$.
    \STATE Initialize $\theta_g$ using pre-trained model on ImageNet, and randomly initialize $\theta_{c_1}, \theta_{c_2}$.
    \REPEAT
    \STATE Randomly sample mini-batch of $B$ source samples and $B$ target samples.
    \STATE Update $\theta_g,\theta_{c_1},\theta_{c_2}$ under the source supervision:
        $\min\limits_{\theta_g,\theta_{c_1},\theta_{c_2}}\frac{1}{B}\sum^{B}_{i=1}\mathcal{L}_{cls}
        (\boldsymbol{x}_i^s,y_i^s)$.
    \STATE Update $\theta_{c_1},\theta_{c_2} $ by maximizing  classifier determinacy disparity:\\
    $\min\limits_{\theta_{c_1},\theta_{c_2}}\frac{1}{B}\sum^{B}_{i=1}\mathcal{L}_{cls}
    (\boldsymbol{x}_i^s,y_i^s)-\alpha\mathcal{L}_{cdd}(\boldsymbol{x}_i^t)$.

    \STATE Update $\theta_g$ by minimizing classifier determinacy disparity:
    $\min\limits_{\theta_{g}}\frac{1}{B}\sum^{B}_{i=1}\alpha\mathcal{L}_{cdd}(\boldsymbol{x}_i^t)$.

    \UNTIL{max\_iteration is reached.}
    \end{algorithmic}
\end{algorithm}

\subsection{Generalization Bound of BCDM}
In this subsection, we further provide the theoretical guarantees for BCDM. Limited by space, all the proofs are given in the supplementary materials.

First, we propose a novel discrepancy to measure the difference between two distributions based on our classifier determinacy disparity.
Assuming $\mathcal{H}$ is a hypothesis space, given two hypotheses $h_1,h_2 \in \mathcal{H}$, we define the expected Determinacy Disparity on some distribution $D$ as: $dis_{D}(h_1,h_2)\triangleq\mathbb{E}_{D}\Gamma(\boldsymbol{p_1},\boldsymbol{p}_2)$.

\begin{Def}
\par
Given a hypothesis set $\mathcal{H}$ and a specific hypothesis $h \in \mathcal{H}$, we define the Determinacy Disparity Discrepancy as:
\begin{small}
\begin{equation}\label{L divergency}
\begin{split}
d_{h,\mathcal{H}}(\mathcal{S},\mathcal{T})
&\triangleq \sup_{h'\in\mathcal{H}}(dis_{\mathcal{T}}(h',h)-dis_{\mathcal{S}}(h',h)) \\
&=\sup_{h'\in\mathcal{H}}(\mathbb{E}_{\mathcal{T}}\Gamma(\boldsymbol{p'},\boldsymbol{p})-\mathbb{E}_{\mathcal{S}}\Gamma(\boldsymbol{p'},\boldsymbol{p})) \,.
\end{split}
\end{equation}
\end{small}
\end{Def}

Compared with $\mathcal{H}\Delta\mathcal{H}-$distance~\cite{A-distance}, the Determinacy Disparity Discrepancy (DDD) we proposed depends upon a hypothesis space $\mathcal{H}$ and a specific classifier $h$, which is easier to be optimized. Moreover, since the DDD is well-designed, we verify that it satisfies symmetrical characteristic and trigonometric inequality through strict proof. Thus, according to the theory proposed in~\cite{DBLP}, we can deduce that the DDD can effectively measure the difference between distributions on domain adaptation, the detail of which is given in the supplementary materials. Besides, we can obtain an upper bound of the expected target error $e_t(h)$ through the DDD.

\begin{prop}\label{prop1}
For every classifier $h \in \mathcal{H}$,
\begin{small}
\begin{equation}\label{eq-target-error-bound2}
\begin{split}
e_t(h)&\leq e_s(h)+d_{h,\mathcal{H}}(\mathcal{S},\mathcal{T})+ \lambda.
\end{split}
\end{equation}
\end{small}where $e_s(h)$ stands for the expected error on source samples, and $\lambda$ denotes the error of ideal joint hypothesis, which can be considered rather small if the hypothesis space is sufficient enough \cite{MDD}.
\end{prop}

Then we introduce the Rademacher complexity which is commonly used in the generalization theory as a measurement of richness for a particular hypothesis space \cite{rademacher}, and provide a generalization bound for our method BCDM based on the Determinacy Disparity Discrepancy defined above.
To begin with, we introduce a Rademacher complexity bound as follows.\par

\begin{lemma} \label{lemma1}
Suppose that $\mathcal{G}$ is a class of function maps $\mathcal{X} \xrightarrow{}{[0,1]}$. Then for any $\delta > 0$, with probability at least $1-\delta$ and sample size $n$, the following holds for all $g\in\mathcal{G}$:
\begin{small}
\begin{equation}\label{Rademacher lemma}
\begin{split}
|\mathbb{E}_Dg-\mathbb{E}_{\widehat{D}}g| \leq 2\mathfrak{R}_{n,D}(\mathcal{G}) + \sqrt{\frac{log\frac{2}{\delta}}{2n}}.
\end{split}
\end{equation}
\end{small}
\end{lemma}
Further, we define a new function class $\mathcal{G}_{\Gamma}\mathcal{H}$ mapping $\mathcal{Z}=\mathcal{X}\times\mathcal{Y}$ to [0,1], which serves as a ``CDD'' version of the symmetric difference hypothesis space $\mathcal{H}\Delta\mathcal{H}$.

\begin{Def}
 Given a hypothesis space $\mathcal{H}$ and a Classifier Disparity Discrepancy function $\Gamma$, the $\mathcal{G}_{\Gamma}\mathcal{H}$ is defined as
\begin{small}
\begin{equation}\label{function family definition}
\begin{split}
\mathcal{G}_{\Gamma}\mathcal{H}=\{\x\xrightarrow{}\Gamma(h_1(\x),h_2(\x))|h_1,h_2 \in \mathcal{H}\}.
\end{split}
\end{equation}
\end{small}
\end{Def}
Based on this definition, we induce a Rademacher complexity bound for the difference between the Determinacy Disparity and its empirical version.
\begin{prop}\label{prop2}
For any $\delta\ge0$, with probability at least $1-\delta$, the following holds for all $h,h'\in \mathcal{H}$:
\begin{small}
\begin{equation}\label{generalization1}
\begin{split}
|dis_{D}(h',h)-dis_{\widehat{D}}(h',h)|\leq2\mathfrak{R}_{n,D}(\mathcal{G}_\Gamma\mathcal{H})+\sqrt{\frac{log\frac{2}{\delta}}{2n}}.
\end{split}
\end{equation}
\end{small}
\end{prop}
Finally, we can come to the final generalization bound by combining the Rademacher complexity bound of CDD and Proposition~\ref{prop1}.

\begin{thm}\label{thm1}
For any $\delta\ge 0$, with probability $1-3\delta$, we have the following generalization bound for any classifier $h \in \mathcal{H}$~:
\begin{small}
\begin{equation}\label{generalization3}
\begin{split}
\!\!err_\mathcal{T}(h)\!\!\leq\!\! err_{\widehat{\mathcal{S}}}(h)\!+\!d_{h,\mathcal{H}}(\widehat{\mathcal{S}},\widehat{\mathcal{T}})\!+\!\lambda
\!+\! \Omega,
\end{split}
\end{equation}
\end{small}%
where $\Omega=2\mathfrak{R}_{n,\mathcal{S}}(\mathcal{G}_\Gamma\mathcal{H})\!+\!2\mathfrak{R}_{n,\mathcal{S}}
(\mathcal{H})\!+\!2\sqrt{\frac{log\frac{2}{\delta}}{2n}}\!+\!2\mathfrak{R}_{m,\mathcal{T}}
(\mathcal{G}_\Gamma\mathcal{H})\!+\!\sqrt{\frac{log\frac{2}{\delta}}{2m}}$. $n,m$ are the amount of samples in source domain and target domain respectively, and $\mathcal{H}$ is the hypothesis set.
\end{thm}

Reviewing our approach, in essence, we minimize the first term $err_{\widehat{\mathcal{S}}}(h)$ in \textbf{Step 1}. Later, the min-max process including \textbf{Step 2} and \textbf{Step 3} implements the minimization of $d_{h,\mathcal{H}}(\widehat{\mathcal{S}},\widehat{\mathcal{T}})$, since $dis_{\mathcal{S}}(h',h)$ is close to zero after the first step. In summary, our method has a rigorous generalization upper bound which can guarantee the effectiveness of BCDM. Through the training process we can achieve the minimization of the major terms in the upper bound, by which the generalization of classifier on target domain will be improved significantly.

\section{Experiment}
\subsection{Dataset and Setup}
We evaluate BCDM against many state-of-the-art
algorithms on four domain adaptation datasets and two semantic segmentation datasets.
\textbf{DomainNet~\cite{DomainNet}} is the largest and hardest dataset to date for visual domain adaptation and consists of about 0.6 million images with 345 classes that spread in 6 distinct domains: Clipart ({\color{blue} clp}), infograph ({\color{blue} inf}), Painting ({\color{blue} pnt}), Quickdraw ({\color{blue} qdr}), Real ({\color{blue} rel}) and Sketch ({\color{blue} skt});
\textbf{VisDA-2017~\cite{visda2017}} is a large-scale synthetic-to-real dataset contains over 280K images across 12 categories;
\textbf{Office-31~\cite{Office31}} is widely  adopted by adaptation methods involving three distinct domains: Amazon (A), DSLR (D) and Webcam (W);
\textbf{ImageCLEF\footnote{http://imageclef.org/2014/adaptation}} is composed of 12 common categories shared by three popular datasets: Caltech-256 (C), ImageNet ILSVRC2012 (I), and PASCAL VOC2012 (P); \textbf{Cityscapes~\cite{cityscapes}} is a real-word dataset with 5,000 urban scenes which are divided into training, validation and test sets; \textbf{GTA5~\cite{GTA5}} contains 24,966 synthesized frames captured from GTA5 game engine. Similar to~\cite{AdaSegNet,SWD}, we use the 19 classes in common with Cityscapes dataset and report the results on the Cityscapes validation set.
Following the standard protocols in~\cite{CDAN,CLAN}, we use the average classification accuracy and PASCAL VOC intersection-over-union~\cite{everingham2015the}, i.e., IoU and mean IoU (mIoU) as evaluation metrics for image classification and semantic segmentation tasks respectively.

\begin{table*}[!htbp]
  \centering
  \caption{Accuracy(\%) on \textbf{DomainNet} for unsupervised domain adaptation. In each sub-table, the column-wise domains are selected as the source domain and the row-wise domains are selected as the target domain. $\ddagger$: Results we obtained using the publicly released codes by the authors without any change and others are borrowed from~\cite{DomainNet}.}
 \resizebox{\textwidth}{!}{
  \setlength{\tabcolsep}{0.5mm}{
    \begin{tabular}{c|ccccccc||c|ccccccclc|ccccccc||c|ccccccc}
    \multicolumn{16}{c}{Accuracy(\%) on DomainNet for UDA (ResNet-50)} &
    & \multicolumn{16}{c}{Accuracy(\%) on DomainNet for UDA (ResNet-101)} \\
    \toprule
    ResNet$^{\ddagger}$ & {\color{blue} clp}  & {\color{blue} inf}    & {\color{blue} pnt}   & {\color{blue} qdr}   & {\color{blue} rel}   & {\color{blue} skt}   & Avg. & MCD$^{\ddagger}$ & {\color{blue} clp}   & {\color{blue} inf}   & {\color{blue} pnt}   & {\color{blue} qdr}   & {\color{blue} rel}   & {\color{blue} skt}   & Avg.  &       & ResNet & {\color{blue} clp}   & {\color{blue} inf}   & {\color{blue} pnt}   & {\color{blue} qdr}   & {\color{blue} rel}   & {\color{blue} skt}   & Avg.  &MCD   & {\color{blue} clp}   & {\color{blue} inf}   & {\color{blue} pnt}   & {\color{blue} qdr}   & {\color{blue} rel}   & {\color{blue} skt}   & Avg. \\
    \cline{1-16}\cline{18-33}  {\color{blue} clp}  & -  & 14.2 & 29.6 & 9.5 & 43.8 & 34.3 & 26.3 & {\color{blue} clp}   & - & 15.4  &  25.5 &  3.3 & 44.6  &  31.2 & 24.0  &    & {\color{blue} clp}   & -  & 19.3  & 37.5  & 11.1  & 52.2  & 41.0  & 32.2  & {\color{blue} clp}   & -     & 14.2  & 26.1  & 1.6   & 45.0  & 33.8  & 24.1 \\
    {\color{blue} inf}   & 21.8  & -  & 23.2 & 2.3   & 40.6  & 20.8  & 21.7  & {\color{blue} inf}   & 24.1 & - & 24.0  & 1.6   & 35.2  & 19.7  &  20.9 &       & {\color{blue} inf}   & 30.2  & -     & 31.2  & 3.6   & 44.0  & 27.9  & 27.4  & {\color{blue} inf}   & 23.6  & -     & 21.2  & 1.5   & 36.7  & 18.0  & 20.2 \\
    {\color{blue} pnt}   &  24.1 & 15.0 & -  & 4.6  &  45.0 &  29.0 &  23.5 & {\color{blue} pnt} & 31.1  & 14.8 & -  & 1.7 &  48.1 & 22.8 & 23.7  &       & {\color{blue} pnt}   & 39.6  & 18.7  & -     & 4.9   & 54.5  & 36.3  & 30.8  & {\color{blue} pnt}   & 34.4  & 14.8  & -     & 1.9   & 50.5  & 28.4  & 26.0 \\
    {\color{blue} qdr}   & 12.2 & 1.5   & 4.9 & -  & 5.6   & 5.7 & 6.0  & {\color{blue} qdr}   & 8.5 & 2.1 & 4.6 & -   & 7.9   & 7.1   & 6.0 &       & {\color{blue} qdr}   & 7.0   & 0.9   & 1.4   & -     & 4.1   & 8.3   & 4.3   & {\color{blue} qdr}   & 15.0  & 3.0   & 7.0   & -     & 11.5  & 10.2  & 9.3 \\
    {\color{blue} rel}   &  32.1 & 17.0  &  36.7 & 3.6   & -   & 26.2  & 23.1  & {\color{blue} rel}  &  39.4  & 17.8  &   41.2 & 1.5 & -  & 25.2  & 25.0 &   & {\color{blue} rel}   & 48.4  & 22.2  & 49.4  & 6.4   & -  & 38.8  & 33.0  & {\color{blue} rel}   & 42.6  & 19.6  & 42.6  & 2.2   & -     & 29.3  & 27.2 \\
    {\color{blue} skt}   &  30.4  & 11.3  & 27.8  & 3.4   & 32.9  & -  & 21.2  & {\color{blue} skt}  & 37.3 & 12.6  & 27.2  & 4.1   & 34.5  & -    & 23.1  &       & {\color{blue} skt}   & 46.9  & 15.4  & 37.0  & 10.9  & 47.0  & -     & 31.4  & {\color{blue} skt}   & 41.2  & 13.7  & 27.6  & 3.8   & 34.8  & -     & 24.2 \\
    Avg.  &  24.1 & 11.8 & 24.4 & 4.7 & 33.6  & 23.2  & 20.3 & Avg.  &  28.1 & 12.5 &  24.5 & 2.4 & 34.1 & 21.2 & 20.5 &       & Avg.  & 34.4  & 15.3  & 31.3  & 7.4   & 40.4  & 30.5  & 26.5  & Avg.  & 31.4  & 13.1  & 24.9  & 2.2   & 35.7  & 23.9  & 21.9 \\
    \cmidrule(lr){1-17}\cmidrule(lr){18-33}
    CDAN$^{\ddagger}$ & {\color{blue} clp}   & {\color{blue} inf}   & {\color{blue} pnt}   & {\color{blue} qdr}   & {\color{blue} rel}   & {\color{blue} skt}   & Avg.  & SWD$^{\ddagger}$ & {\color{blue} clp}   & {\color{blue} inf}   & {\color{blue} pnt}   & {\color{blue} qdr}   & {\color{blue} rel}   & {\color{blue} skt}   & Avg.  &       & CDAN$^{\ddagger}$  & {\color{blue} clp}   & {\color{blue} inf}   & {\color{blue} pnt}   & {\color{blue} qdr}   & {\color{blue} rel}   & {\color{blue} skt}   & Avg.  & SWD$^{\ddagger}$  & {\color{blue} clp}   & {\color{blue} inf}   & {\color{blue} pnt}   & {\color{blue} qdr}   & {\color{blue} rel}   & {\color{blue} skt}   & Avg. \\
    \cmidrule(lr){1-17}\cmidrule(lr){18-33}
    {\color{blue} clp}   & -  &  13.5 & 28.3 & 9.3 & 43.8 & 30.2 & 25.0  & {\color{blue} clp}   & -   & 14.7 & 31.9 & 10.1 & 45.3 & 36.5 & 27.7 & & {\color{blue} clp} & - & 17.8 & 35.7 & 15.3 & 51.3 & 37.2 & 31.4  & {\color{blue} clp}   & -     & 16.6 & 35.3 & 12.8  & 48.7 & 41.0 & 30.9 \\
    {\color{blue} inf}   &  18.9 & - & 21.4 & 1.9  & 36.3  & 21.3  & 20.0  & {\color{blue} inf}   & 22.9 & -  & 24.2 & 2.5 & 33.2 & 21.3 & 20.0 &  & {\color{blue} inf}   & 25.4 & -  & 28.9 & 5.8 & 38.2 & 22.8 & 24.2 & {\color{blue} inf}   & 26.9 & - & 27.6 & 2.7 & 38.1 & 25.4 & 24.1\\
    {\color{blue} pnt}  & 29.6  & 14.4 & - & 4.1 & 45.2 & 27.4 & 24.2 & {\color{blue} pnt}  & 33.6 & 15.3 & - & 4.4 & 46.1 & 30.7 & 26.0 &  & {\color{blue} pnt} & 37.1 & 17.9 & - & 7.9 & 51.4  & 34.0  & 29.7  & {\color{blue} pnt}   & 37.3 & 16.9 & - & 5.9 & 48.7 & 34.6 & 28.7 \\
    {\color{blue} qdr}  & 11.8  & 1.2   & 4.0  & -  & 9.4   & 9.5 & 7.2 & {\color{blue} qdr}   & 15.5 & 2.2 & 6.4 & - & 11.1 & 10.2 & 9.1 &  & {\color{blue} qdr}  & 20.5 & 2.3 & 7.7  & - & 14.6 & 12.6 & 11.5 & {\color{blue} qdr}   & 19.3 & 3.0 & 8.1 & - & 14.2 & 13.3 & 11.6 \\
    {\color{blue} rel}   & 36.4 & 18.3  & 40.9 & 3.4 & - & 24.6 & 24.7 & {\color{blue} rel}  & 41.2 & 18.1 & 44.2 & 4.6 & - & 31.6 & 27.9 &  & {\color{blue} rel}   & 43.6 & 19.4 & 46.1 & 8.3 & - & 33.2 & 30.1 & {\color{blue} rel} & 47.0 & 19.9 & 47.1 & 6.1 & - & 36.8 & 31.4 \\
    {\color{blue} skt}  & 38.2 & 14.7  & 33.9 & 7.0 & 36.6  & -  & 26.1 & {\color{blue} skt} & 44.2 & 15.2 & 37.3 & 10.3 & 44.7 & -  & 30.3  &       & {\color{blue} skt}   & 45.4 & 18.3 & 40.4 & 14.5 & 48.3 & -  &  33.4 & {\color{blue} skt} & 48.8 & 17.3 & 41.1 & 12.2 & 49.1 & - & 33.7 \\
    Avg.  & 27.0 & 12.4 & 25.7 & 5.1 & 34.3 & 22.6 & 21.2 & Avg.  & 31.5 & 13.1 & 28.8 & 6.4 & 36.1 & 26.1 & 23.6 &  & Avg.  & 34.4 & 15.1 & 31.7 & 10.4 & 40.8 & 27.9 & 26.7 & Avg. & 35.9 & 14.7 & 31.8 & 7.9 & 39.8 & 30.2 & 26.7 \\

\cmidrule(lr){1-17}\cmidrule(lr){18-33}
    BNM$^{\ddagger}$ & {\color{blue} clp}   & {\color{blue} inf}   & {\color{blue} pnt}   & {\color{blue} qdr}   & {\color{blue} rel}   & {\color{blue} skt}   & Avg.  & \textbf{BCDM} & {\color{blue} clp}   & {\color{blue} inf}   & {\color{blue} pnt}   & {\color{blue} qdr}   & {\color{blue} rel}   & {\color{blue} skt}   & Avg.  &       & BNM$^{\ddagger}$ & {\color{blue} clp}   & {\color{blue} inf}   & {\color{blue} pnt}   & {\color{blue} qdr}   & {\color{blue} rel}   & {\color{blue} skt}   & Avg.  &\textbf{BCDM} & {\color{blue} clp}   & {\color{blue} inf}   & {\color{blue} pnt}   & {\color{blue} qdr}   & {\color{blue} rel}   & {\color{blue} skt}   & Avg. \\
    \cmidrule(lr){1-17}\cmidrule(lr){18-33}
    {\color{blue} clp} & - & 12.1 & 33.1 & 6.2 & 50.8 & 40.2 & 28.5 & {\color{blue} clp} & - & 17.2 & 35.2 & 10.6 & 50.1 & 40.0 & 30.6 &  & {\color{blue} clp} & - & 19.4 & 35.6 & 16.1 & 49.8 & 36.3 & 31.4 & {\color{blue} clp} & - & 19.9 & 38.5 & 15.1 & 53.2 & 43.9 & 34.1  \\
    {\color{blue} inf} & 26.6 & -  & 28.5 & 2.4 & 38.5 & 18.1 & 22.8 & {\color{blue} inf} & 29.3 & - & 29.4 & 3.8  & 41.3  & 25.0  & 25.8 &   & {\color{blue} inf} & 24.6 & - & 27.8 & 7.9 & 35.0 & 22.0 & 23.5 & {\color{blue} inf} & 31.9 & - & 32.7  & 6.9 & 44.7  & 28.5 & 28.9  \\
    {\color{blue} pnt} & 39.9 & 12.2 & - & 3.4 & 54.5 & 36.2 & 29.2 & {\color{blue} pnt} & 39.2  & 17.7  & -  & 4.8  & 51.2 & 34.9 & 29.6 &  & {\color{blue} pnt} & 36.0 & 20.2 & - & 9.7 & 51.8 & 34.2 & 30.4 & {\color{blue} pnt} & 42.5 & 19.8  & - & 7.9  & 54.5 & 38.5  &  32.6 \\
    {\color{blue} qdr} & 17.8 & 1.0 & 3.6 & - & 9.2 & 8.3 & 8.0 & {\color{blue} qdr} & 19.4 & 2.6  & 7.2  & - & 13.6 & 12.8 & 11.1  &  & {\color{blue} qdr} & 21.3 & 3.8 & 10.5 & - & 14.0 & 12.9 & 12.5 & {\color{blue} qdr} & 23.0 & 4.0 & 9.5 & - & 16.9 &  16.2 & 13.9   \\
    {\color{blue} rel}   & 48.6 & 13.2 & 49.7 & 3.6 & -  & 33.9 & 29.8 & {\color{blue} rel}   & 48.2  & 21.5  & 48.3  & 5.4 & - & 36.7  & 32.0  &   & {\color{blue} rel} & 43.4 & 21.7 & 47.0 & 9.9 & - & 32.9 & 31.0 & {\color{blue} rel} & 51.9  & 24.9 & 51.2 & 8.7 & -  & 40.6 & 35.5  \\
    {\color{blue} skt} & 54.9 & 12.8 & 42.3 & 5.4 & 51.3 & - & 33.3 & {\color{blue} skt}   & 50.6  & 17.3  & 41.9  & 10.6  & 49.0  & - & 33.9  &  & {\color{blue} skt} & 43.1 & 19.1 & 39.5 & 15.6 & 47.0 & - & 32.7 & {\color{blue} skt} & 53.7 & 20.5 & 46.0 & 13.1 & 53.4 & -  &37.1    \\
    Avg.  & 37.6 & 10.3 & 31.4 & 4.2 & 40.9 & 27.3 &  25.3 & Avg.  & 37.3  & 15.3  & 32.4  & 7.0 & 41.0 & 29.9 & {\bf27.2}  &  & Avg.  & 33.7 & 16.8& 32.1 & 11.8 & 39.6 &  27.7 & 26.9 & Avg.  & 40.6 & 17.8 & 35.6 & 10.3 & 44.3 & 33.5 & \textbf{30.4}  \\
    \bottomrule
    \end{tabular}
    }
    }
\label{tab:domainnet}
\end{table*}

\subsection{Implementation Details}
For image classification, we use pre-trained ResNet-50/101~\cite{resnet} from ImageNet~\cite{deng2009imagenet} as the base feature extractor $G$ and replace the last fully-connected (FC) layer with a bottleneck layer as~\cite{CDAN}. The architecture of task-specific classifier is identical to~\cite{MCD}: three FC layers with random initialization is attached to the bottleneck layer. 
The FC layers are trained from the scratch with learning rate 10 times that of the feature generator layers. We adopt Stochastic Gradient Descent optimizer (SGD)~\cite{SGD} with learning rate $3 \times 10^{-4}$, momentum 0.9 and weight decay $5 \times 10^{-4}$.  To schedule the learning rate, we follow the annealing procedure mentioned in~\cite{MADA}. Besides, we exploit the entropy loss as~\cite{MCD,RTN} to make the training procedure more stable in all experiment.

For semantic segmentation, we utilize the DeepLab-v2~\cite{chen2018deeplab} with ResNet-101 pre-trained on ImageNet as the base architecture $G$. Also, Atrous Spatial Pyramid Pooling (ASPP) is used for classifier and applied on the $conv5$ feature outputs. Following~\cite{AdaSegNet,CLAN}, sampling rates are fixed as \{6, 12, 18, 24\} and we modify the stride and dilation rate of the last layers to produce denser feature maps. To train the network, we use the SGD with Nesterov acceleration where initial learning rate is set as $2.5 \times 10^{-4}$ with momentum 0.9 and weight decay $10^{-4}$. And we use the polynomial annealing procedure as mentioned in~\cite{chen2018deeplab}. Regarding the training procedure, the network is first trained with $\mathcal{L}_{cls}$ for 20k iterations and then fine-tune using Algorithm~\ref{alg:CDD} for 40k iterations. 

\begin{table*}[htbp]
  \centering
  \caption{Accuracy(\%) on \textbf{VisDA-2017} for unsupervised DA (ResNet-101).}
    \resizebox{\textwidth}{!}{
    \begin{tabular}{c|cccccccccccc|c}
    \toprule
    Method & {\color{blue} plane} & {\color{blue} bcycl} & {\color{blue} bus}   & {\color{blue} car}   & {\color{blue} horse} & {\color{blue} knife} & {\color{blue} mcycl} & {\color{blue} person} & {\color{blue} plant} & {\color{blue} sktbrd} & {\color{blue} train} & {\color{blue} truck} & Avg. \\
    \hline
    ResNet~\cite{resnet} & 55.1  & 53.3  & 61.9  & 59.1  & 80.6  & 17.9  & 79.7  & 31.2  & 81.0  & 26.5  & 73.5  & 8.5   & 52.4  \\
    DAN~\cite{DAN} & 87.1 & 63.0  & 76.5  & 42.0  & 90.3 & 42.9  & 85.9  & 53.1  & 49.7  & 36.3  & 85.8 & 20.7  & 61.1  \\
    MinEnt~\cite{entropy_minimization} & 80.3 & 75.5 & 75.8 & 48.3 & 77.9 & 27.3 & 69.7 & 40.2 & 46.5 & 46.6 & 79.3 & 16.0 & 57.0 \\
    DANN~\cite{DANN} & 81.9  & 77.7 & 82.8  & 44.3  & 81.2  & 29.5  & 65.1  & 28.6  & 51.9  & 54.6 & 82.8  & 7.8   & 57.4  \\
    CDAN~\cite{CDAN} & 85.2 & 66.9 & 83.0 & 50.8 & 84.2 & 74.9 & 88.1 &  74.5 & 83.4 & 76.0  &  81.9  &  38.0 & 73.9 \\
    JADA~\cite{JADA}  & 91.9  & 78.0 & 81.5  & 68.7  & 90.2  & 84.1  & 84.0  & 73.6  & 88.2  & 67.2 & 79.0  & 38.0  & 77.0  \\
    TPN~\cite{TPN} & 93.7 & 85.1 & 69.2 & 81.6 & 93.5 & 61.9 & 89.3 & 81.4 & 93.5 & 81.6 & 84.5 & \textbf{49.9} & 80.4 \\
    AFN~\cite{AFN} & 93.6 & 61.3 & 84.1 & 70.6 & \textbf{94.1} & 79.0 & \textbf{91.8} & 79.6 & 89.9 & 55.6 & 89.0 & 24.4 & 76.1 \\
    BNM~\cite{BNM} & 89.6 & 61.5 & 76.9 & 55.0 & 89.3 & 69.1 & 81.3 & 65.5 & 90.0 & 47.3  & \textbf{89.1} & 30.1 & 70.4 \\
    \hline
    MCD~\cite{MCD} & 87.0  & 60.9  &  83.7  & 64.0  & 88.9  & 79.6  & 84.7  & 76.9 & 88.6 & 40.3  & 83.0  & 25.8 & 71.9  \\
    SWD~\cite{SWD}  &  90.8 &  82.5  &  81.7  &  70.5  &  91.7  &  69.5  &  86.3  &  77.5  &  87.4  &  63.6  &  85.6  &  29.2  &  76.4 \\
    STAR~\cite{STAR} & 95.0 & 84.0 & \textbf{84.6} & 73.0 & 91.6 & 91.8 & 85.9 & 78.4 & 94.4 & 84.7 & 87.0 & 42.2 & 82.7 \\
    \hline
    \textbf{BCDM} & \textbf{95.1}  & \textbf{87.6}  & 81.2 & \textbf{73.2} &  92.7  & \textbf{95.4} & 86.9 & \textbf{82.5}  & \textbf{95.1} & \textbf{84.8}  & 88.1  & 39.5  & \textbf{83.4} \\
    \bottomrule
    \end{tabular}
    }
  \label{tab:visda}

\end{table*}

\begin{table*}[htbp] 
	\centering
	\caption{Classification Accuracy (\%) on \textbf{Office-31} and  \textbf{ImageCLEF} Datasets (ResNet-50).}
  \resizebox{\textwidth}{!}{
	\begin{tabular}{c|ccccccc|ccccccc}
		\toprule
        \multirow{2}{*}{Method}    & \multicolumn{7}{c|}{\textbf{Office-31}} & \multicolumn{7}{c}{\textbf{ImageCLEF}}  \\
		\cline{2-8}\cline{9-15} & {\color{blue} A $\rightarrow$ W}  &  {\color{blue} D $\rightarrow$ W}  &  {\color{blue} W $\rightarrow$ D}  &  {\color{blue} A $\rightarrow$ D}  &  {\color{blue} D $\rightarrow$ A}  &  {\color{blue} W $\rightarrow$ A}  & Avg. & {\color{blue} I $\rightarrow$ P} & {\color{blue} P $\rightarrow$ I} & {\color{blue} I $\rightarrow$ C} & {\color{blue} C $\rightarrow$ I} & {\color{blue} C $\rightarrow$ P} & {\color{blue} P $\rightarrow$ C} & Avg. \\
		\hline
    ResNet~\cite{resnet} & 68.4$\pm$0.2 & 96.7$\pm$0.1 & 99.3$\pm$0.1 & 68.9$\pm$0.2 & 62.5$\pm$0.3 & 60.7$\pm$0.3 & 76.1 & 74.8 & 83.9 & 91.5 & 78.0 & 65.5 & 91.2 & 80.7 \\
    DAN~\cite{DAN}   & 80.5$\pm$0.4 & 97.1$\pm$0.2 & 99.6$\pm$0.1 & 78.6$\pm$0.2 & 63.6$\pm$0.3 & 62.8$\pm$0.2 &  80.4 & 74.5 & 82.2 & 92.8 & 86.3 & 69.2 & 89.8 & 82.5 \\
    MinEnt~\cite{entropy_minimization} & 86.8$\pm$0.2 & \textbf{98.6$\pm$0.1} & \textbf{100.0$\pm$.0} & 88.7$\pm$0.3 & 67.2$\pm$0.5 & 63.4$\pm$0.4 & 84.1 & 76.2 & 85.7 & 93.5 & 83.5 & 69.3 & 89.7 & 83.0 \\
    DANN~\cite{DANN} & 82.0$\pm$0.4 & 96.9$\pm$0.2 & 99.1$\pm$0.1 & 79.7$\pm$0.4 & 68.2$\pm$0.4 & 67.4$\pm$0.5 &  82.2 & 75.0 & 86.0 & 96.2 & 87.0 & 74.3 & 91.5 & 85.0 \\
    CDAN~\cite{CDAN} & 94.1$\pm$0.1 &  98.6$\pm$0.1 & \textbf{100.0$\pm$.0} & 92.9$\pm$0.2 & 71.0$\pm$0.3 & 69.3$\pm$0.3 & 87.7 & 77.7 & 90.7 & \textbf{97.7} & 91.3 & 74.2 & 94.3 & 87.7 \\
    JADA~\cite{JADA} & 90.5$\pm$0.1 & 97.5$\pm$0.3 & \textbf{100.0$\pm$.0} & 88.2$\pm$0.1 & 70.9$\pm$0.5 & 70.6$\pm$0.4 &  86.1  & 78.2 & 90.1 & 95.9 & 90.8 & 76.8 & 94.1 & 87.7 \\
    AFN~\cite{AFN} & 90.1$\pm$0.8 & 98.6$\pm$0.2 & 99.8$\pm$.0 & 90.7$\pm$0.5 & 73.0$\pm$0.2 & 70.2$\pm$0.3 & 87.1 & 79.3 & \textbf{93.3} & 96.3 & \textbf{91.7} & 77.6 & 95.3 & 88.9 \\
    BNM~\cite{BNM} & 91.5$\pm$.0 & 98.5$\pm$.0 & \textbf{100.0$\pm$.0}  & 90.3$\pm$.0 & 70.9$\pm$.0 & 71.6$\pm$.0 & 87.1 & 77.2 & 91.2 & 96.2 &  \textbf{91.7} & 75.7 & \textbf{96.7} & 88.1 \\
   \hline
   MCD~\cite{MCD}  & 88.6$\pm$0.2 & 98.5$\pm$0.1 & \textbf{100.0$\pm$.0} & 92.2$\pm$0.2 & 69.5$\pm$0.1 & 69.7$\pm$0.3 &  86.5 & 77.3 & 89.2 & 92.7 & 88.2 & 71.0 & 92.3 & 85.1 \\
   SWD~\cite{SWD} & 90.4$\pm$0.4 & \textbf{98.7$\pm$0.2} &  \textbf{100.0$\pm$.0} & \textbf{94.7$\pm$0.4} & 70.3$\pm$0.2 & 70.5$\pm$0.5 & 87.4 & 78.1 & 89.6 & 95.2 & 89.3  & 73.4 & 92.8 & 86.4 \\
   STAR~\cite{STAR} & 92.6$\pm$0.1 & \textbf{98.7$\pm$0.1} & \textbf{100.0$\pm$.0} & 93.2$\pm$0.1 & 71.4$\pm$0.3 & 70.8$\pm$0.1 & 87.8  & 78.8 & 90.5 & 96.2 & 91.2 & 75.5 & 93.8 & 87.7\\
   \hline
   \textbf{BCDM}    & \textbf{95.4$\pm$0.3} & 98.6$\pm$0.1 & \textbf{100.0$\pm$.0} &  93.8$\pm$0.3  &  \textbf{73.1$\pm$0.3}  &  \textbf{73.0$\pm$0.2}  &  \textbf{89.0} & \textbf{79.5} &  93.2 & 96.8 & 91.3 & \textbf{78.9} & 95.8 & \textbf{89.3} \\
		\bottomrule
	\end{tabular}
    }
  \label{tab:office31-clef}
\end{table*}

\begin{table*}[!htbp]
  \centering
  \caption{Semantic segmentation performance mIoU (\%) on \textbf{Cityscapes} validation set.}
  \resizebox{\textwidth}{!}{
  \begin{tabular}{c|ccccccccccccccccccc|c}
    \toprule
    Method & \rotatebox{60}{{\color{blue} road}} & \rotatebox{60}{{\color{blue} side.}} & \rotatebox{60}{{\color{blue} buil.}} & \rotatebox{60}{{\color{blue} wall}} & \rotatebox{60}{{\color{blue} fence}} & \rotatebox{60}{{\color{blue} pole}} & \rotatebox{60}{{\color{blue} light}} & \rotatebox{60}{{\color{blue} sign}} & \rotatebox{60}{{\color{blue} veg}} & \rotatebox{60}{{\color{blue} terr.}} & \rotatebox{60}{{\color{blue} sky}} & \rotatebox{60}{{\color{blue} pers.}} & \rotatebox{60}{{\color{blue} rider}} & \rotatebox{60}{{\color{blue} car}} & \rotatebox{60}{{\color{blue} truck}} & \rotatebox{60}{{\color{blue} bus}} & \rotatebox{60}{{\color{blue} train}} & \rotatebox{60}{{\color{blue} mbike}} & \rotatebox{60}{{\color{blue} bike}} & mIoU \\
    
    \hline
    NoAdaptation (ResNet-101)  & 75.8 & 16.8 & 77.2 & 12.5 & 21.0 & 25.5 & 30.1 & 20.1 & 81.3 & 24.6 & 70.3 & 53.8 & 26.4 & 49.9 & 17.2 & 25.9 & 6.5 & 25.3 & \textbf{36.0} & 36.6 \\
    AdaSegNet~\cite{AdaSegNet} & 86.5 & 36.0 & 79.9 & 23.4 & 23.3 & 23.9 & 35.2 & 14.8 & 83.4 & 33.3 & 75.6 & 58.5 & 27.6 & 73.7 & 32.5 & 35.4 & 3.9 & 30.1 & 28.1 & 42.4 \\
    AdvEnt~\cite{AdvEnt} & 89.9 & 36.5 & 81.6 & 29.2 & 25.2 & 28.5 & 32.3 & 22.4 & 83.9 & 34.0 & 77.1 & 57.4 & 27.9 & 83.7 & 29.4 & 39.1 & 1.5 & 28.4 & 23.3 & 43.8 \\
    SWD~\cite{SWD} & \textbf{92.0} & \textbf{46.4} & 82.4 & 24.8 & 24.0 & \textbf{35.1} & 33.4 & \textbf{34.2} & 83.6 & 30.4 & 80.9 & 56.9 & 21.9 & 82.0 & 24.4 & 28.7 & 6.1 & 25.0 & 33.6 & 44.5 \\
    CLAN~\cite{CLAN} & 87.0 & 27.1 & 79.6 & 27.3 & 23.3 & 28.3 & 35.5 & 24.2 & 83.6 & 27.4 & 74.2 & 58.6 & 28.0 & 76.2 & \textbf{33.1} & 36.7 & 6.7 & \textbf{31.9} & 31.4 & 43.2 \\
    STAR~\cite{STAR} & 88.4 & 27.9 & 80.8 & 27.3 & \textbf{25.6} & 26.9 & 31.6 & 20.8 & 83.5 & 34.1 & 76.6 & \textbf{60.5} & 27.2 & 84.2 & 32.9 & 38.2 & 1.0 & 30.2 & 31.2 & 43.6 \\
    \hline
    \textbf{BCDM} & 90.5 & 37.3 & \textbf{83.7} & \textbf{39.2} & 22.2 & 28.5 & \textbf{36.0} & 17.0 & \textbf{84.2} & \textbf{35.9} & \textbf{85.8} & 59.1 & \textbf{35.5} & \textbf{85.2} & 31.1 & \textbf{39.3} & \textbf{21.1} & 26.7 & 27.5 & \textbf{46.6} \\
    \bottomrule
  \end{tabular}
  }
  \label{tab:seg}
\end{table*}

\begin{figure*}[!htbp]
  \centering  
  \includegraphics[width=0.98\textwidth]{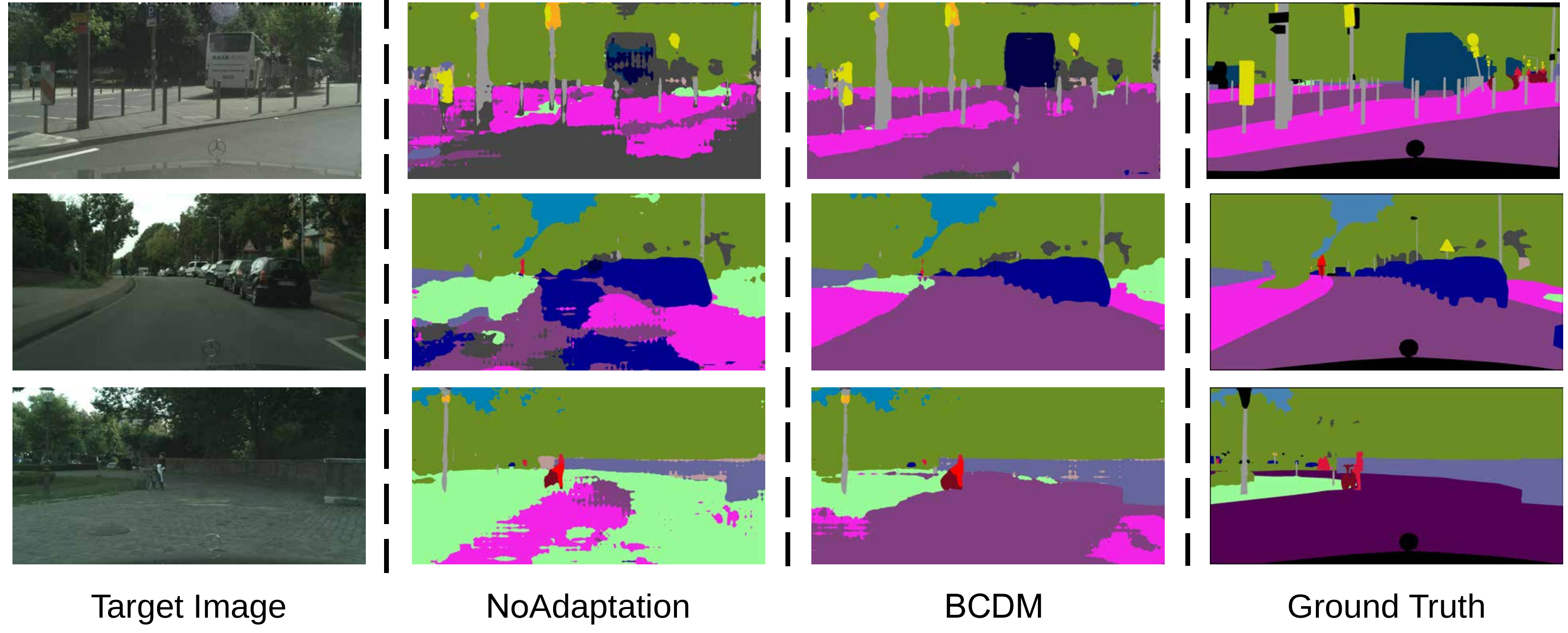}
  \caption{Qualitative examples of adaptation segmentation results for GTA5 $\rightarrow$ Cityscapes setting. For each target image, we show results before adaptation, the proposed method, and ground truth.}
  \label{Fig_seg_visulation_1}
\end{figure*}

\subsection{Experimental Results}

\textbf{DomainNet.} As illustrated in Table~\ref{tab:domainnet}, our BCDM significantly outperforms the comparison methods (i.e., CDAN, BNM, SWD, MCD) in terms of final mean accuracy. Note that the mainstream domain adaptation baseline, i.e., MCD, suffers from negative transfer~\cite{survey}. An interpretation is that the MCD model is category agnostic and will bring side effects of the deterioration of the representation discriminability, especially when there exist notable domain gaps and a large number of categories on this dataset.
In spite of that, our method still yields large improvement over other baselines for the scenarios of using both ResNet-50/101 backbones, which highlights its superiority in aligning distinct domains with large class variation and its suitability for this imbalanced dataset~\cite{DomainNet}.

\textbf{VisDA-2017.} Table~\ref{tab:visda} compares various methods with the pretrained ResNet-101 backbone. Clearly, we outperform other approaches by large margins on 7 out of 12 tasks. Furthermore, our BCDM brings a rise of improvement by $\textbf{31.0}\%$ over the source only model (ResNet-101). Compared with MCD and SWD, which explicitly utilize adversarial learning with bi-classifier to align features across source domain and target domain, our model achieves extra gains of $\textbf{11.5}\%$ and $\textbf{7.0}\%$ respectively. Besides, the BCDM surpasses them on bicycle, knife, and sktbrd classes by large margins. Moreover, we achieve decent gains compared to the best baseline STAR. The considerable results reflect that the proposed classifier determinacy disparity metric is effective in associating two different but related distributions of data.

\textbf{Office-31 and ImageCLEF.} Experimental results are reported in Table~\ref{tab:office31-clef}. On both datasets, our BCDM obtains superior results over other popular adaptation methods and achieves the best average accuracies ($\textbf{89.0}\%$ on Office-31 and $\textbf{89.3}\%$ on ImageCLEF). The results show that BCDM is beneficial to promote the adaptation capability, especially on the difficult scenarios where the baseline accuracy is relatively low, e.g., D $\rightarrow$ A, W $\rightarrow$ A and C $\rightarrow$ P. 

\textbf{GTA5 $\rightarrow$ Cityscapes.}
Unlike image classification, segmentation task generally requires more effort of human labor to annotate each pixel in an image. Here we extend our BCDM to perform domain adaptive  semantic segmentation task. We present quantitative results of transferring GTA5
to Cityscapes in Table~\ref{tab:seg} with the best results highlighted in bold. Even with a large domain shift between synthetic-to-real scenes, the model trained on our BCDM has been shown to outperform the model trained on source only (NoAdaptation) by a significant gain of $\textbf{10.0}\%$. Additionally, our method consistently outperforms other recent approaches that utilize bi-classifier~\cite{SWD,CLAN,STAR}. We also provide more qualitative examples in Fig.~\ref{Fig_seg_visulation_1}, where our method produces less noisy regions and yields good segmentation predictions with more details.

\begin{figure*}[!htbp]
  \centering  
    \subfigure[ResNet-50]{
      \label{Fig_moon_res}
      \includegraphics[width=0.23\textwidth]{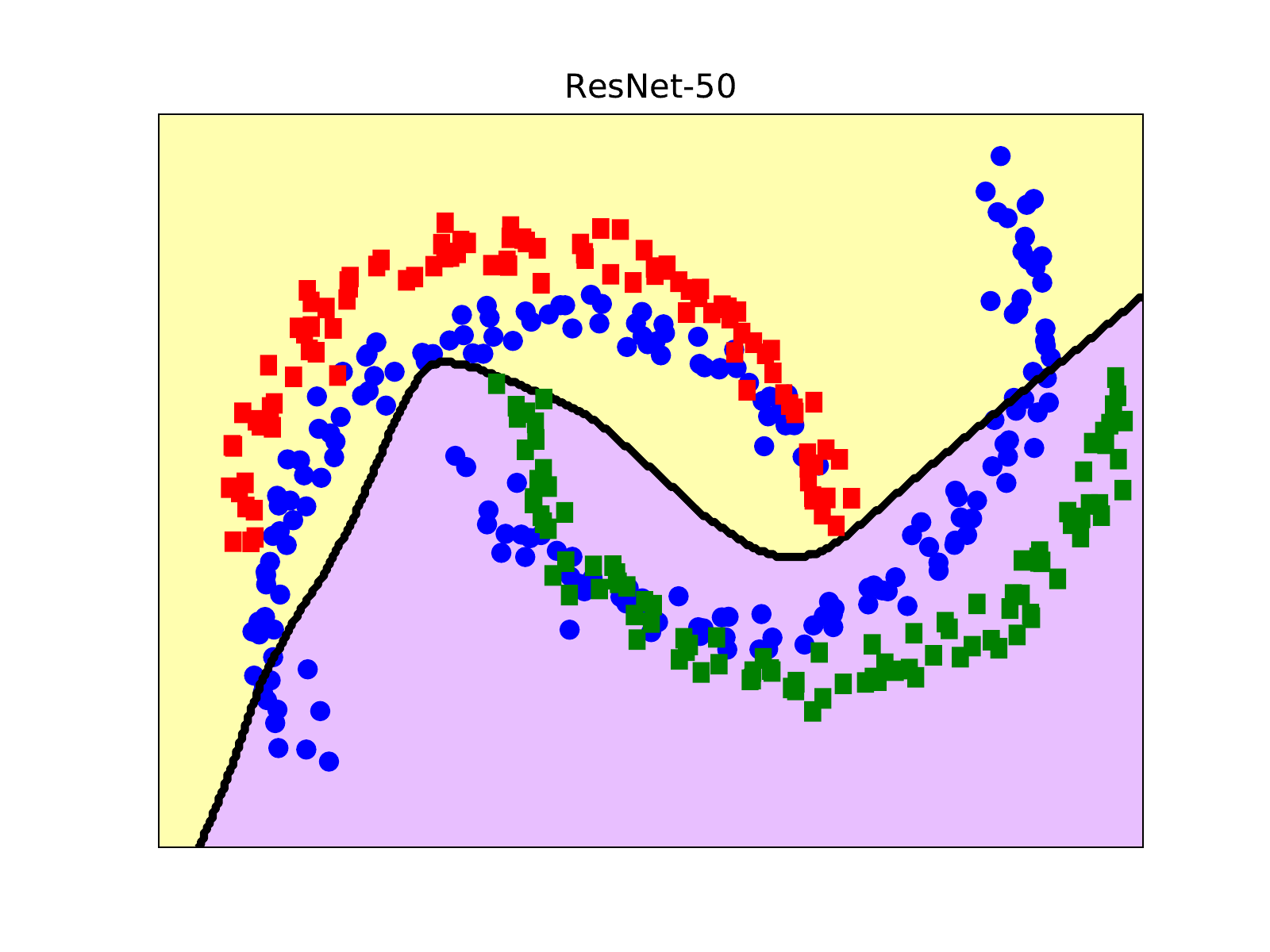}
    }
    \subfigure[MCD]{
      \label{Fig_moon_mcd}
      \includegraphics[width=0.23\textwidth]{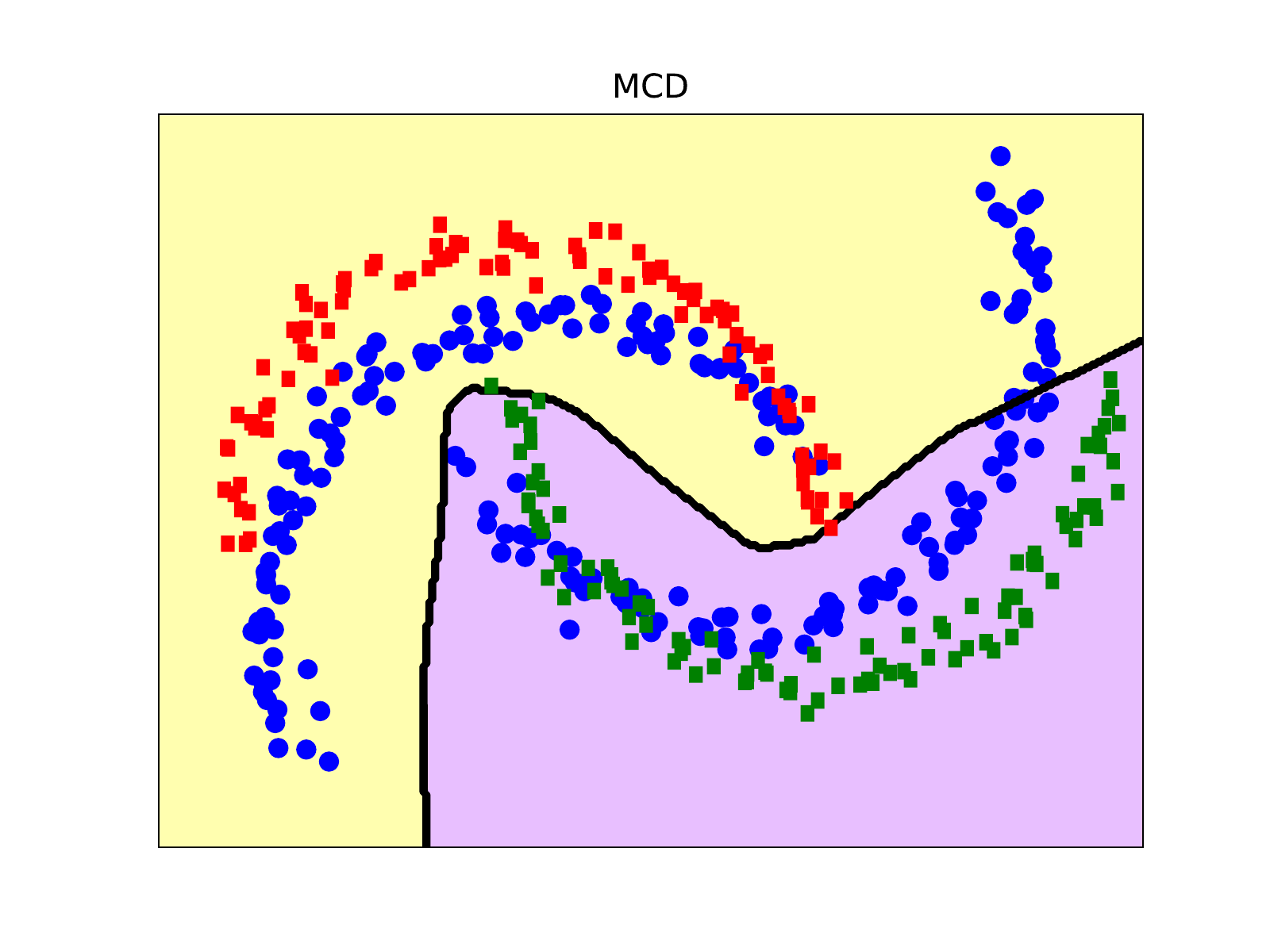}
    }
    \subfigure[SWD]{
      \label{Fig_moon_swd}
      \includegraphics[width=0.23\textwidth]{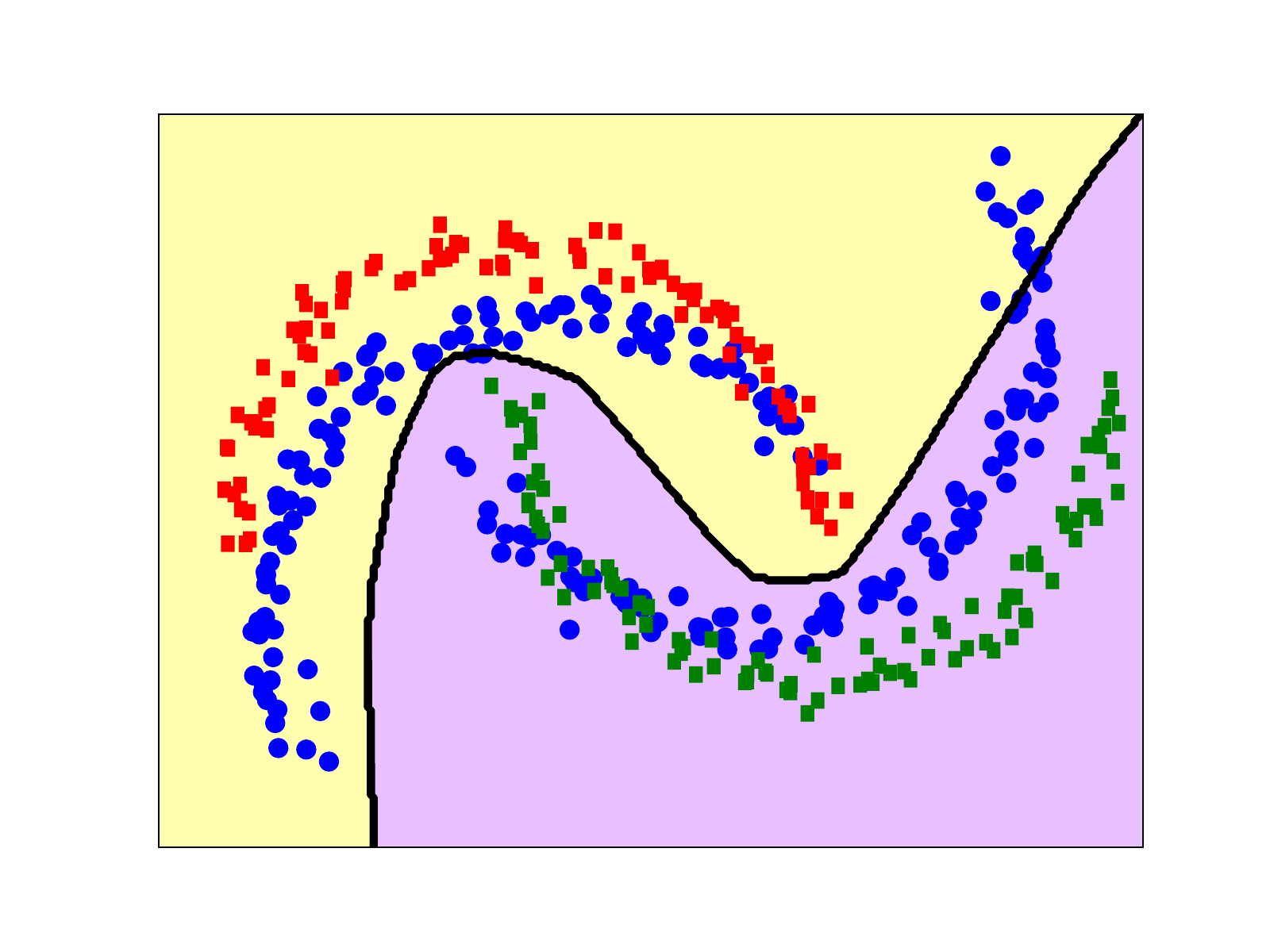}
    }
    \subfigure[BCDM]{
      \label{Fig_moon_bcdm}
      \includegraphics[width=0.23\textwidth]{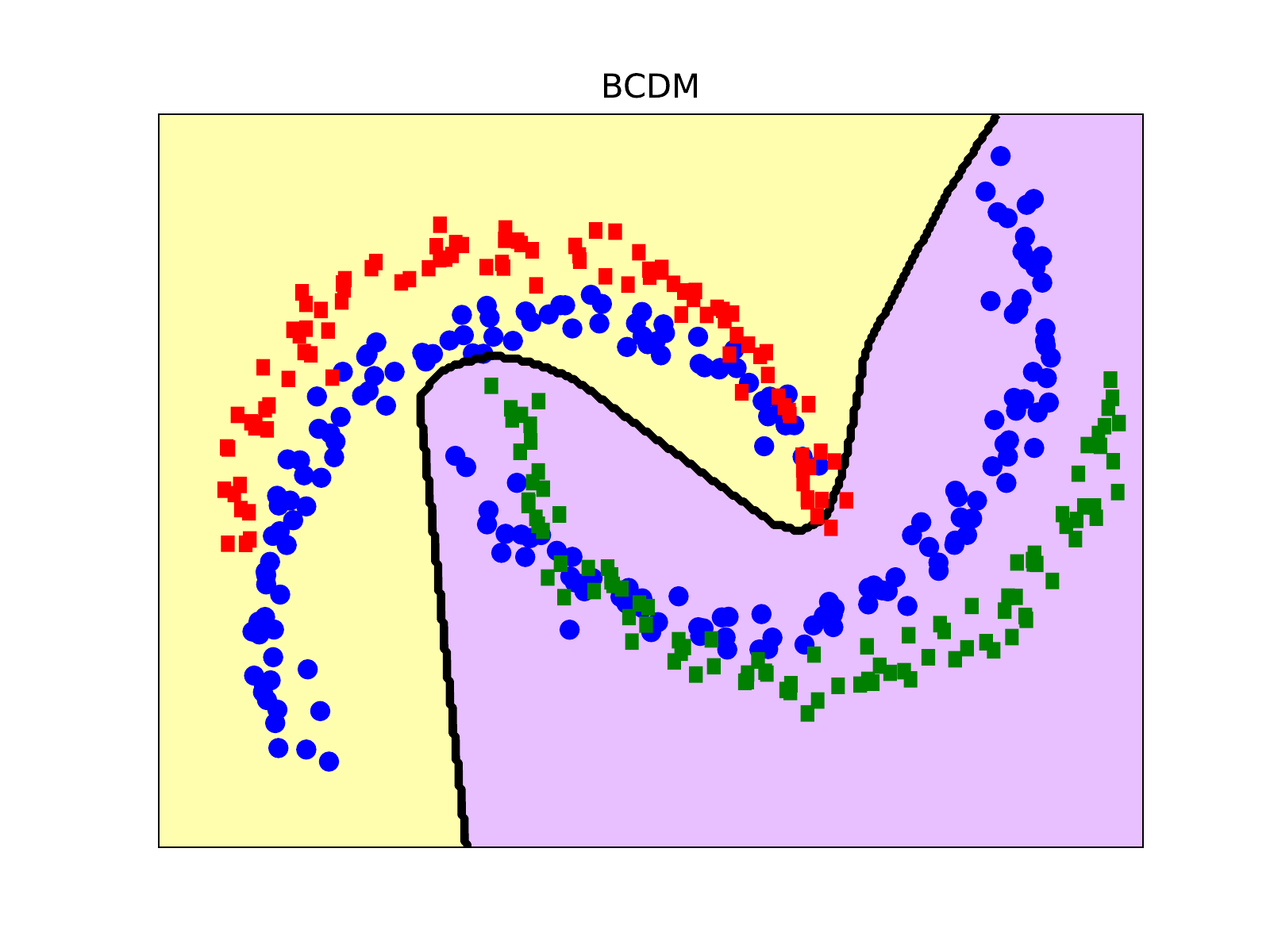}
    }
  \caption{Comparisons of four decision boundaries on a toy dataset. Red points and green points denote the class 0 and 1 of source data, respectively. Blue points are target data generated from the same distribution as the source data but with domain shift by rotation. The yellow and magenta regions are classified as class 0 and 1 by the final decision boundary, respectively.}
  \label{Fig_moons}
\end{figure*}

\begin{figure*}[!htbp]
  \centering  
    \subfigure[$\A$-distance]{
      \label{Fig_a_distance}
      \includegraphics[width=0.19\textwidth]{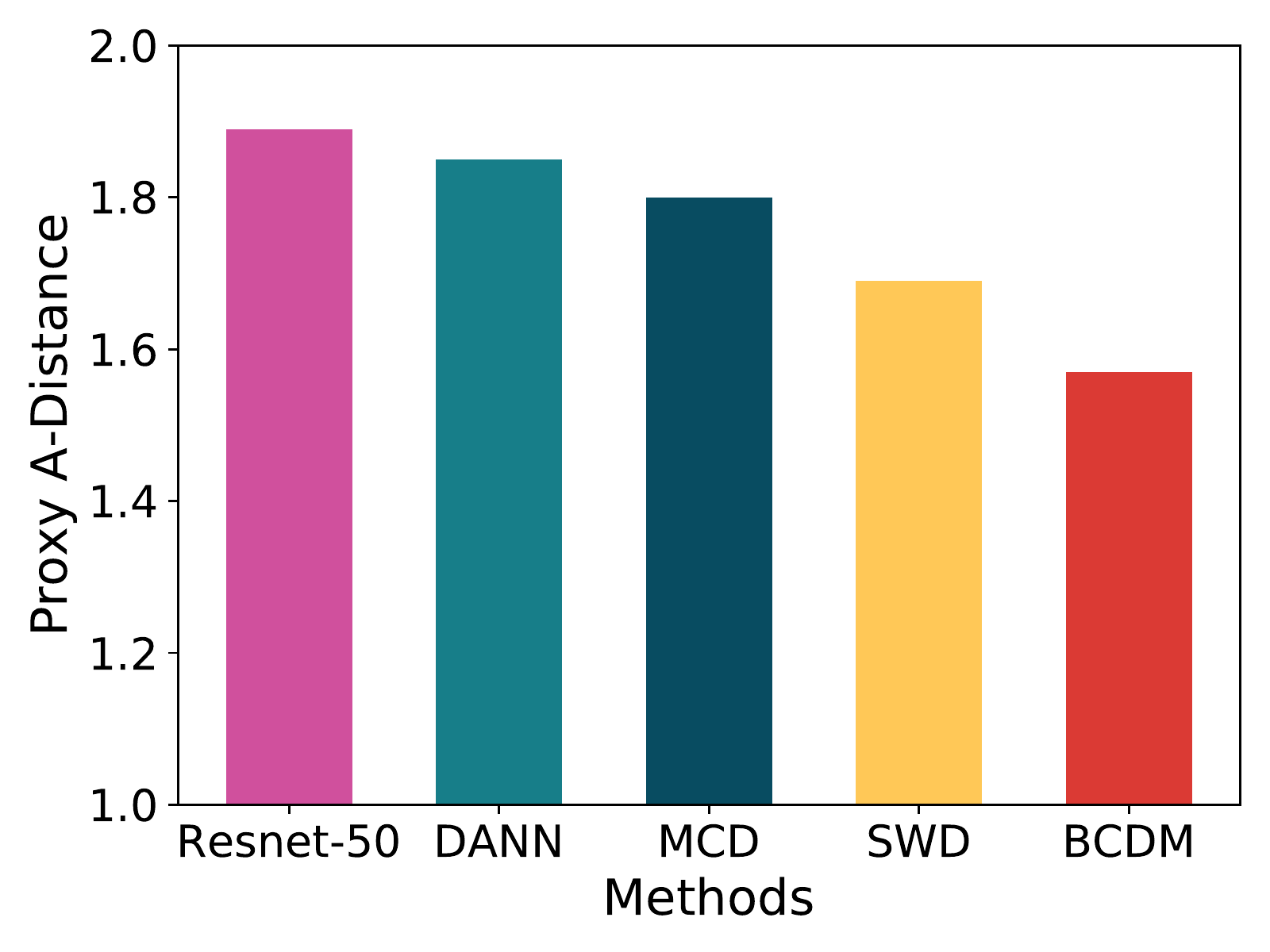}
    }
    \subfigure[Convergence]{
      \label{Fig_convergence_office}
      \includegraphics[width=0.185\textwidth]{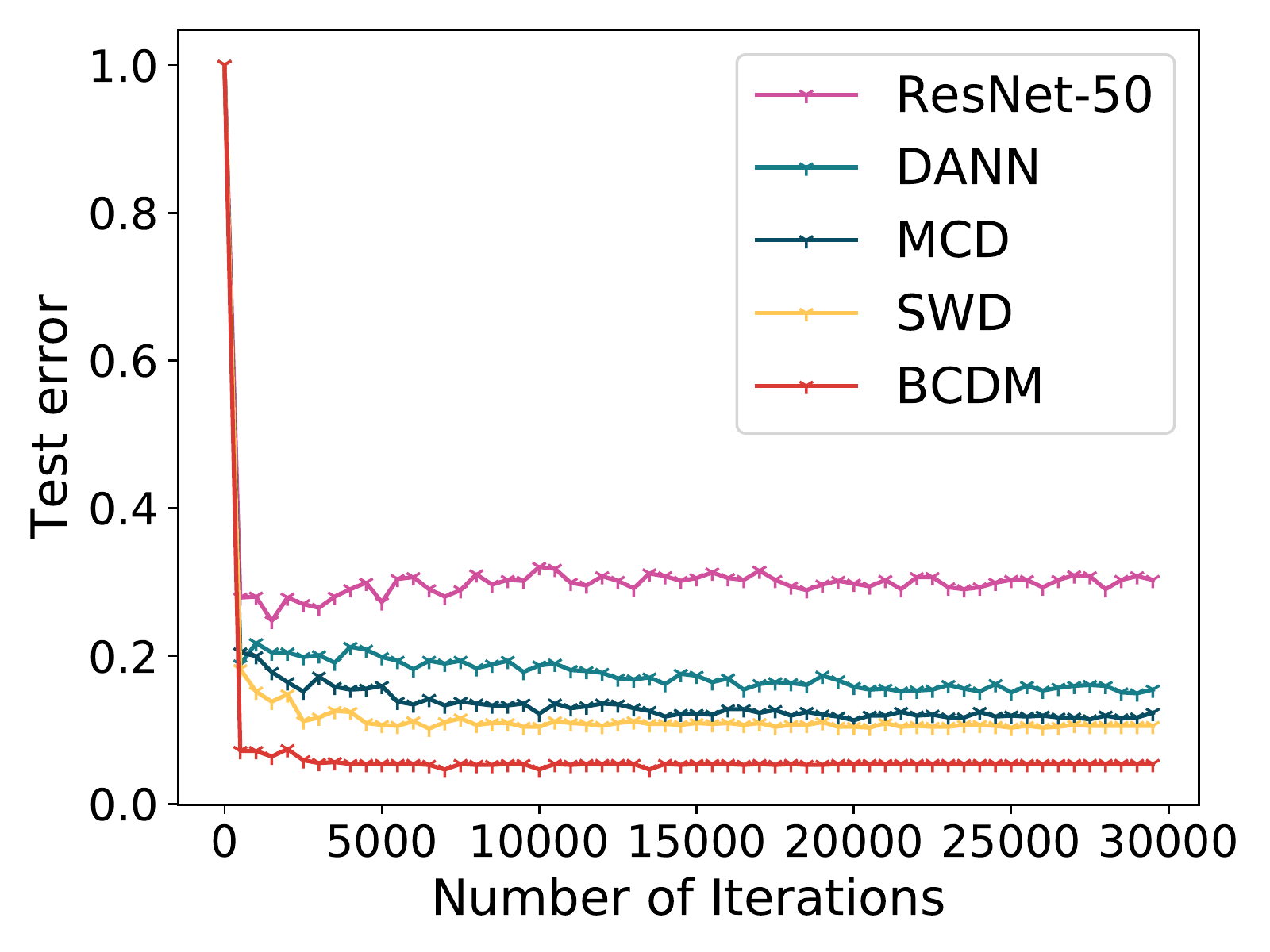}
    }
    \subfigure[SVD analysis]{
      \label{Fig_svd_analysis}
      \includegraphics[width=0.185\textwidth]{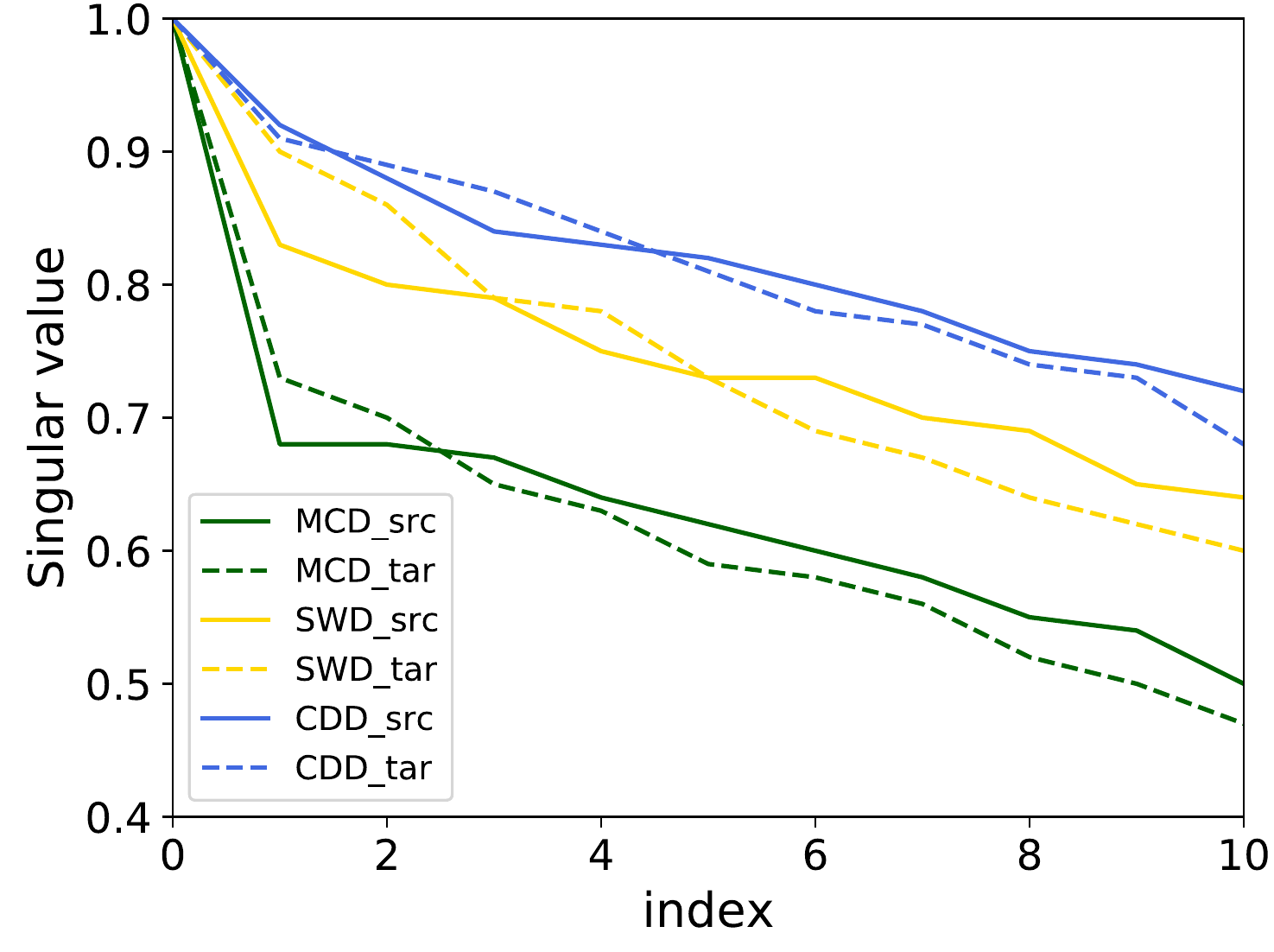}
    }
    \subfigure[MCD]{
      \label{Fig_tsne_mcd}
      \includegraphics[width=0.185\textwidth]{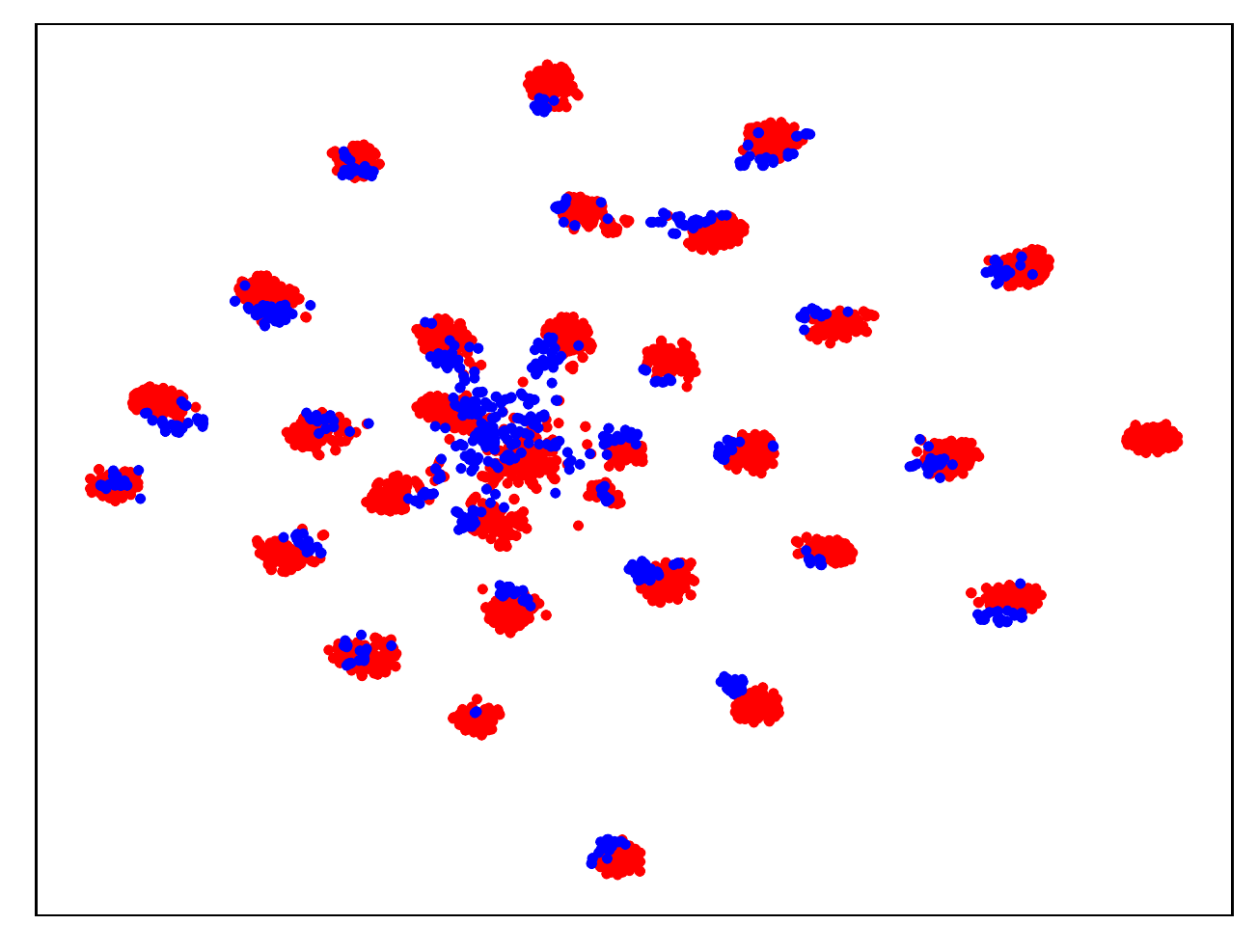}
    }
    \subfigure[BCDM]{
      \label{Fig_tsne_bcdm}
      \includegraphics[width=0.185\textwidth]{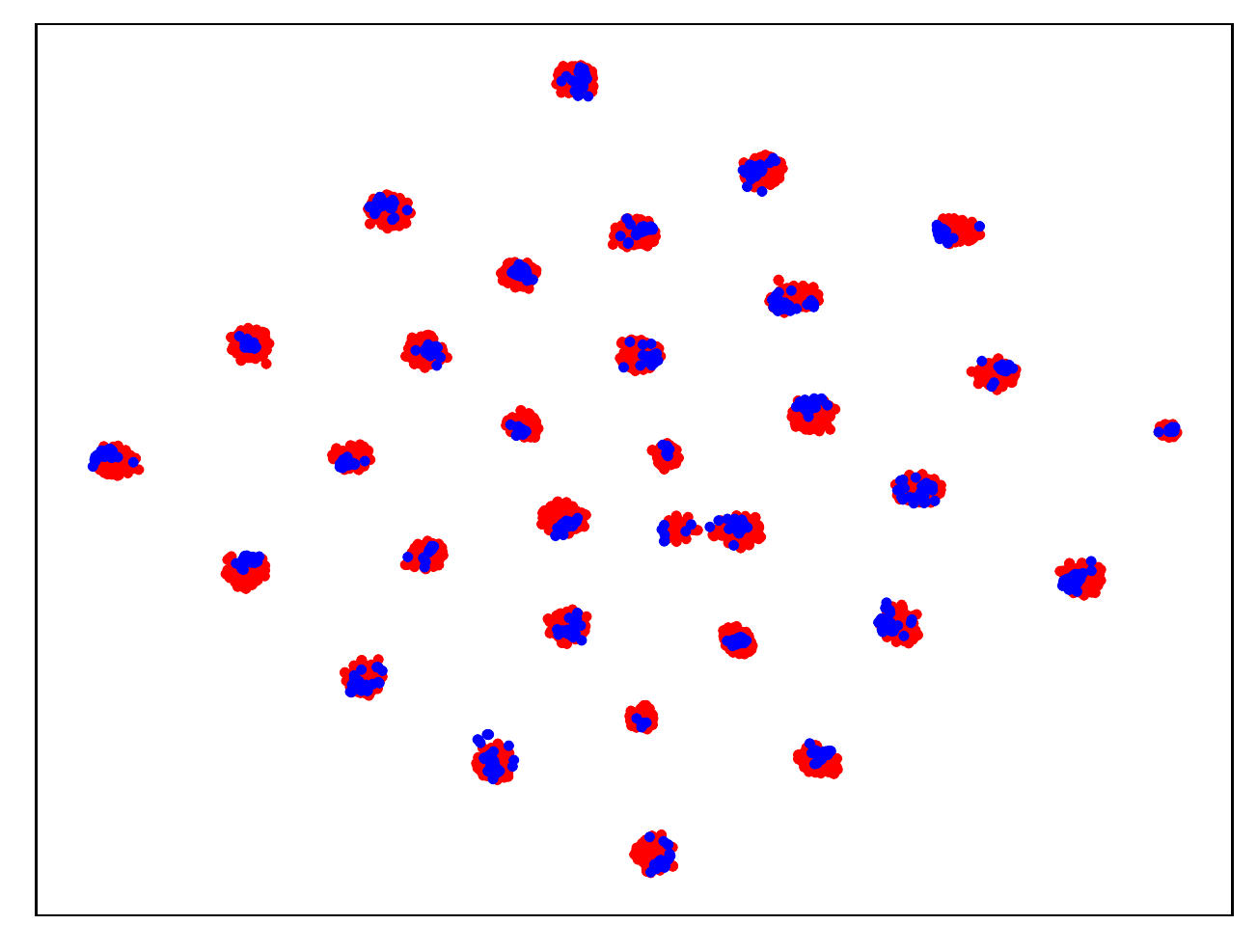}
    }
  \caption{Visualization of $\A$-distance, convergence, SVD analysis, t-SNE on task A (red) $\rightarrow$ W (blue) of \textbf{Office-31}.}
\end{figure*}
\subsection{Insight Analysis}

\textbf{Toy Experiment.}
We perform a toy experiment on the inter twinning moons 2D problems~\cite{pedregosa2011scikit-learn} to provide analysis on the learned decision boundaries of ResNet-50, MCD, SWD and BCDM. As indicated in Fig.~\ref{Fig_moons}, we generate two classes of source samples, one is an upper moon, the other is a lower moon, labeled 0 and 1, respectively. For the target samples, we enhance domain gap by rotating 30 degrees of the distribution of the source samples to generate them. 
\begin{table}[htbp]
	\centering
  \caption{Generalized Results (\%) on \textbf{Office-31} Dataset.}
  \resizebox{0.48\textwidth}{!}{
  \begin{tabular}{l|c|cccccc|c|c}
    \toprule
	  Method &Metric& {\color{blue} A $\rightarrow$ W}  &  {\color{blue} D $\rightarrow$ W}  &  {\color{blue} W $\rightarrow$ D}  &  {\color{blue} A $\rightarrow$ D}  &  {\color{blue} D $\rightarrow$ A}  &  {\color{blue} W $\rightarrow$ A}  & Avg. &$\varDelta$ \\
		\hline
    STAR  &$\ell_1$& 92.6& \textbf{98.7} & \textbf{100.0} & 93.2 & 71.4 & 70.8 & 87.8& - \\
    + SWD& $swd$ & 93.5 & \textbf{98.7} & \textbf{100.0} & 93.4 & 72.1 & 71.7 & 88.2 & 0.4 $\uparrow$ \\
    + BCDM &$cdd$& \textbf{95.1} & \textbf{98.7} & \textbf{100.0} & \textbf{94.1} & \textbf{73.5} & \textbf{73.7} & \textbf{89.2} & \textbf{1.4} $\uparrow$\\
		\bottomrule
	\end{tabular}
    }
  \label{tab:generalized_results}
\end{table}
We generate 100 samples per class as the training samples. Both feature generator and classifiers employ a 3-layered fully-connected network. After convergence, ResNet-50 (Fig.~\ref{Fig_moon_res}) correctly classifies the source samples under the supervision of source labels but does badly in target samples. MCD (Fig~\ref{Fig_moon_mcd}) is able to reduce the distribution shift between domains to some extent, but it is not good enough to perform properly in the whole target samples. SWD (Fig.~\ref{Fig_moon_swd} is capable of perfectly aligning its decision boundary in yellow region (class 0) but fails in magenta region (class 1). In contrast, BCDM (Fig.~\ref{Fig_moon_bcdm}) achieves the best matching to the target domain samples and obtains the right decision boundary for all classes.

\textbf{Generalization Results.}
To further demonstrate the generalization of our metric, classifier determinacy disparity (CDD) can be used as an additive module to the latest bi-classifier UDA method STAR~\cite{STAR}, and we compare its performance against $\ell_1$ and Wasserstein metric (SWD). As shown in Table~\ref{tab:generalized_results}, CDD achieves the best performance compared to  $\ell_1$ and SWD. Specifically, our proposed metric improves the baseline approach by up to 1.4\%.

\textbf{Proxy $\A$-distance.}
As shown in Fig.~\ref{Fig_a_distance}, we compute the proxy $\A$-distance of ResNet-50, DANN, MCD, SWD and BCDM task-specific feature representations for task A~$\rightarrow$~W of Office-31. We observe that BCDM has the smallest $\A$-distance among the compared approaches
which justifies that the representations of BCDM are more indistinguishable and can reduce the domain divergence effectively.

\textbf{Convergence.}
We present the test error convergence curve with respect to the number of iterations on task A~$\rightarrow$~W as shown in Fig~\ref{Fig_convergence_office}. It can be observed that BCDM achieves comparable convergence with smaller errors.

\textbf{SVD analysis.} We plot the singular values(max-normalized) of features extracted by MCD, SWD and BCDM in Fig.~\ref{Fig_svd_analysis}. BCDM keeps high values while successfully reducing the big difference between the largest and the rest. It implies that more dimensions corresponding to smaller
singular values pose positive influence on classification, which intuitively boosts the feature discriminability.

\textbf{t-SNE Visualization.} We visualize features of the last $fc$-layer of ResNet backbone using t-SNE~\cite{tsne} in Fig.~\ref{Fig_tsne_mcd}-\ref{Fig_tsne_bcdm}. The target data are not aligned well with source data using MCD, while BCDM can learn highly discriminative features and keep clear class boundaries.

\section{Conclusion}

In this paper, we investigate that the representation discriminability and classifier determinacy imply more transferability, which is a general exploration for unsupervised domain adaptation. 
Accordingly, we introduce a new metric, classifier determinacy disparity (CDD), to felicitously evaluate the discrepancy across distinct target predictions. With the merits of CDD and its theoretical guarantees, we propose Bi-Classifier Determinacy Maximization (BCDM) to adversarially optimize the CDD loss, achieving  significant performance. Extensive results convincingly demonstrate that the BCDM outperforms the state-of-the-arts on a wide range of unsupervised domain adaptation scenarios including both image classification and semantic segmentation tasks.

\section*{Acknowledgments}
This work was supported by the National Natural Science Foundation of China (61902028).

\section*{Ethics Statement}
Similar to some automation-related research which has an impact on job loss, our work, which focuses on the application of automated annotation for an unseen domain, is no exception. Specifically, this work has a positive impact on society and community to save the cost and time of data annotation, which can greatly improve the efficiency. However, this work suffers from some negative consequences, which worth us carrying out further researches to explore. For example, more jobs of classification or object detection for rare or variable conditions may be canceled. 
Furthermore, we should be cautious of the result of the failure of the system, which could render people believe that classification was unbiased. Still, it might be not, which might be misleading. Finally, this work does leverage biases in the data, which is the primary task of this work.


\bibliography{Reference_AAAI2021}
\clearpage

\section{Supplementary Materials}
\subsection{Properties of CDD}

\begin{prop}\label{properties of CDD}
Suppose $\mathcal{H}$ is a hypothesis space. Given two hypotheses $h_1,h_2\in\mathcal{H}$ and their probabilistic outputs $\boldsymbol{p}_1,\boldsymbol{p}_2$, the Classifier Determinacy Disparity (CDD) we defined is a pseudometric.
\end{prop}

\begin{proof}
Firstly, $\Gamma(\boldsymbol{p}_1,\boldsymbol{p}_2)$ is non-negative since all the elements of the Bi-classifier Prediciton Relevance Matrix $\mathbf{A}$ is non-negative, i.e., $A_{ij} \ge 0(i,j\in\{1,2,\ldots K \})$. Thus,
\begin{small}
\begin{equation}\nonumber
    \Gamma(\boldsymbol{p}_1,\boldsymbol{p}_2)= \sum_{i,j=1}^KA_{ij}-\sum_{i=1}^KA_{ii}=\sum_{i\neq j}^KA_{ij}\ge0.
\end{equation}
\end{small}%
Secondly, $\Gamma(\boldsymbol{p}_1,\boldsymbol{p}_2)=0$ iff. $\boldsymbol{p}_1=\boldsymbol{p}_2$ and each of the probabilistic output is a one-hot vector. This is exactly the reason why CDD can guarantee the representation discriminability.\\
Thirdly, $\Gamma(\boldsymbol{p}_1,\boldsymbol{p}_2)$ is symmetric, since
\begin{small}
\begin{equation}\nonumber
\begin{split}
\Gamma(\boldsymbol{p}_1,\boldsymbol{p}_2)&=\sum_{i,j=1}^KA_{ij}-\sum_{i=1}^KA_{ii}\\
&=\sum_{i,j=1}^KA_{ji}-\sum_{i=1}^KA_{ii}\\
&=\Gamma(\boldsymbol{p}_2,\boldsymbol{p}_1).
\end{split}
\end{equation}
\end{small}%
Lastly, $\Gamma(\boldsymbol{p}_1,\boldsymbol{p}_2)$ satisfies triangle inequality. For any $h_1,h_2\in\mathcal{H}$, we have:
\begin{small}
\begin{equation}\nonumber
\begin{split}
&\Gamma(\boldsymbol{p}_1,\boldsymbol{p}_2)-(\Gamma(\boldsymbol{p}_1,\boldsymbol{p}')+\Gamma(\boldsymbol{p}',\boldsymbol{p}_2))\\
&=\sum_{i,j=1}^K(\boldsymbol{p}_1\boldsymbol{p}^{\top}_2)_{ij}-\sum_{i=1}^K(\boldsymbol{p}_1\boldsymbol{p}^{\top}_2)_{ii}-[\sum_{i,j=1}^K(\boldsymbol{p}_1\boldsymbol{p'}^{\top})_{ij}\\
&-\sum_{i=1}^K(\boldsymbol{p}_1\boldsymbol{p'}^{\top})_{ii}+\sum_{i,j=1}^K(\boldsymbol{p'}\boldsymbol{p}^{\top}_2)_{ij}-\sum_{i=1}^K(\boldsymbol{p'}\boldsymbol{p}^{\top}_2)_{ii} ]\\
&=\sum_{i=1}^K(\boldsymbol{p}_1\boldsymbol{p'}^{\top})_{ii}+\sum_{i=1}^K(\boldsymbol{p'}\boldsymbol{p}^{\top}_2)_{ii}+\sum_{i=1}^K(\boldsymbol{p}_1\boldsymbol{p}^{\top}_2)_{ii}-1\\
&=\sum_{i=1}^K(\boldsymbol{p}_1\boldsymbol{p'}^{\top})_{ii}+\sum_{i=1}^K(\boldsymbol{p'}\boldsymbol{p}^{\top}_2)_{ii}+\sum_{i=1}^K(\boldsymbol{p}_1\boldsymbol{p}^{\top}_2)_{ii}-\sum_{i=1}^K(\boldsymbol{p'}\boldsymbol{1}^{\top})_{ii}\\
&\leq \sum_{i=1}^K(\boldsymbol{p}_1\boldsymbol{1}^{\top})_{ii}+\sum_{i=1}^K(\boldsymbol{1}\boldsymbol{p}^{\top}_2)_{ii}+\sum_{i=1}^K(\boldsymbol{p}_1\boldsymbol{p}^{\top}_2)_{ii}-\sum_{i=1}^K(\boldsymbol{1}\boldsymbol{1}^{\top})_{ii}\\
&=\sum_{i=1}^K(\boldsymbol{p}_1(\boldsymbol{1}-\boldsymbol{p}_2)^{\top}-\boldsymbol{1}(\boldsymbol{1}-\boldsymbol{p}_2)^{\top})_{ii}\\
&=\sum_{i=1}^K((\boldsymbol{p}_1-\boldsymbol{1})(\boldsymbol{1}-\boldsymbol{p}_2)^{\top})_{ii}\\
& \leq 0 \,,
\end{split}
\end{equation}
\end{small}%
where $\boldsymbol{1}\in\mathbb{R}^{K\times1}$ and each element of $\boldsymbol{1}$ is 1.\\
Thus, $\Gamma(\boldsymbol{p}_1,\boldsymbol{p}_2)$ is a pseudometric.
\end{proof}

\begin{prop}\label{properties of DDD}
Let $\mathcal{P}$ be the space of probability distributions over the domain $\mathcal{X}$ and $\mathcal{H}$ be a hypothesis space. Given any classifier $h\in\mathcal{H}$, the induced Determinacy Disparity Disparity $d_{h,\mathcal{H}}(P,Q)$ is a pseudometric on $\mathcal{P}$.
\end{prop}
\begin{proof}
Obviously, $d_{h,\mathcal{H}}(P,Q)$ is non-negative and $d_{h,\mathcal{H}}(P,P)=0$ holds for any $P\in\mathcal{P}$ by definition.\\
Moreover,
\begin{small}
\begin{equation}\nonumber
\begin{split}
  d_{h,\mathcal{H}}(P,Q)&=\sup_{h'\in\mathcal{H}}(dis_Q(h',h)-dis_P(h',h))\\
&=\sup_{h'\in\mathcal{H}}(dis_P(h',h)-dis_Q(h',h))\\
&=d_{h,\mathcal{H}}(Q,P).
\end{split}
\end{equation}
\end{small}%
Thus, $d_{h,\mathcal{H}}(P,Q)$ satisfies symmetrical characteristic.\\
Finally, for any distribution $P,Q,R$, we have:
\begin{small}
\begin{equation}\nonumber
\begin{split}
d_{h,\mathcal{H}}(P,Q)&=\sup_{h'\in\mathcal{H}}(dis_Q(h',h)-dis_P(h',h))\\
& \leq \sup_{h'\in\mathcal{H}}(dis_Q(h',h)-dis_R(h',h))\\
& +\sup_{h''\in\mathcal{H}}(dis_R(h'',h)-dis_P(h'',h))\\
& = d_{h,\mathcal{H}}(P,R)+d_{h,\mathcal{H}}(R,Q).
\end{split}
\end{equation}
\end{small}%
In summary, $d_{h,\mathcal{H}}(P,Q)$ is a pseudometric on $\mathcal{P}$.
\end{proof}

~\\

\begin{lemma}[Definition 4 $\&$ Theorem8 of \cite{DBLP}.]
\label{upper bound}
Suppose $\mathcal{H}$ is a hypothesis space, for any $h\in\mathcal{H}$, we have the following upper bound for target error risk:
\begin{small}
\begin{equation}
\begin{split}
e_Q(h)&\leq e_P(h)+disc_L(P,Q)+\lambda \,, 
\end{split}
\end{equation}
\end{small}
\begin{small}
\begin{equation}
\begin{split}
disc_L(P,Q)&=\sup_{h,h'\in\mathcal{H}} |\mathbb{E}_QL(h',h)-\mathbb{E}_PL(h',h)|\,,
\end{split}
\end{equation}
\end{small}%
where L should be a bounded function satisfying symmetry and triangle inequality.
\end{lemma}
\begin{prop}\label{upper bound for CDD}
Given two different distributions of source and target domains $\mathcal{S},\mathcal{T}$, for any classifier $h$, we have
\begin{small}
\begin{equation}
\begin{split}
e_t(h)&\leq e_s(h)+d_{h,\mathcal{H}}(\mathcal{S},\mathcal{T})+ \lambda \,,
\end{split}
\end{equation}
\end{small}%
where $\lambda=\lambda(\mathcal{H},\mathcal{S},\mathcal{T})$ is independent of $h$.
\end{prop}
\begin{proof}
Since $\Gamma(\boldsymbol{p}_1,\boldsymbol{p}_2)\in[0,1]$, and it satisfies symmetry and triangle inequality which is proven above. Combine with Lemma \ref{upper bound}, we can directly get the upper bound.
\end{proof}

\subsection{Generalization Bounds with CDD}
\begin{lemma}[Rademacher Generalization Bound, Theorem
3.1 of \cite{rademacher}.] \label{rademacher bound}
Suppose that $\mathcal{G}$ is a class of function maps $\mathcal{X} \xrightarrow{}{[0,1]}$. Then for any $\delta > 0$, with probability at least $1-\delta$ and sample size $n$, the following holds for all $g\in\mathcal{G}$:
\begin{small}
\begin{equation}
\begin{split}
|\mathbb{E}_Dg-\mathbb{E}_{\widehat{D}}g| \leq 2\mathfrak{R}_{n,D}(\mathcal{G}) + \sqrt{\frac{log\frac{2}{\delta}}{2n}}.
\end{split}
\end{equation}
\end{small}%
\end{lemma}

\begin{prop}\label{DD rademacher bound}
Suppose $\mathcal{H}$ is a hypothesis space, $\mathcal{D}$ is a distribution and $\mathcal{\widehat{D}}$ is the corresponding empirical version which contains $n$ data points sampled from $\mathcal{D}$. For any $\delta\ge0$, with probability at least $1-\delta$, the following holds for all $h,h'\in \mathcal{H}$:
\begin{small}
\begin{equation}
\begin{split}
|dis_{D}(h',h)-dis_{\widehat{D}}(h',h)|\leq2\mathfrak{R}_{n,D}(\mathcal{G}_\Gamma\mathcal{H})+\sqrt{\frac{log\frac{2}{\delta}}{2n}}.
\end{split}
\end{equation}
\end{small}%
\end{prop}

\begin{proof}
\begin{small}
\begin{equation}\nonumber
\begin{split}
  & \sup_{h,h' \in \mathcal{H}}|dis_D(h',h)-dis_{\widehat{D}}(h',h)|\\
  & = \sup_{h,h' \in \mathcal{H}}|\mathbb{E}_{D}(\Gamma(\boldsymbol{p'},\boldsymbol{p}))-\mathbb{E}_{ \widehat{D}}(\Gamma(\boldsymbol{p'},\boldsymbol{p}))|\\
  & = \sup_{g \in \mathcal{G}_\Gamma\mathcal{H}}|\mathbb{E}_{D}g-\mathbb{E}_{ \widehat{D}}g| \,.\\
\end{split}
\end{equation}
\end{small}%
With Lemma \ref{rademacher bound}, we could know that
\begin{small}
\begin{equation}\nonumber
\begin{split}
\sup_{g \in \mathcal{G}_\Gamma\mathcal{H}}|\mathbb{E}_{D}g-\mathbb{E}_{ \widehat{D}}g|\leq2\mathfrak{R}_{n,D}(\mathcal{G}_\Gamma\mathcal{H})+\sqrt{\frac{log\frac{2}{\delta}}{2n}}.
\end{split}
\end{equation}
\end{small}%
\end{proof}

\begin{thm}\label{thm1}
For any $\delta\ge 0$, with probability $1-3\delta$, we have the following generalization bound for any classifier $h \in \mathcal{H}$~:
\begin{small}
\begin{equation}
\begin{split}
err_{\mathcal{T}}(h)& \leq err_{\widehat{\mathcal{S}}}(h)+d_{h,\mathcal{H}}(\widehat{\mathcal{S}},\widehat{\mathcal{T}})+\lambda\\
& + 2\mathfrak{R}_{n,\mathcal{S}}(\mathcal{G}_\Gamma\mathcal{H})+2\mathfrak{R}_{n,\mathcal{S}}(\mathcal{H})+2\sqrt{\frac{log\frac{2}{\delta}}{2n}}\\
& + 2\mathfrak{R}_{m,\mathcal{T}}(\mathcal{G}_\Gamma\mathcal{H})+\sqrt{\frac{log\frac{2}{\delta}}{2m}} \,,
\end{split}
\end{equation}
\end{small}%
where $\lambda = \min_{h'\in\mathcal{H}}(err_\mathcal{S}(h')+err_\mathcal{T}(h'))$. $n,m$ are the amount of samples in source and target domain respectively, and $\mathcal{H}$ is the hypothesis set.
\end{thm}

\begin{proof}
Consider the difference of expected and empirical terms on the right-hand side:
\begin{small}
\begin{equation}\nonumber
\begin{split}
&\sup_{h\in\mathcal{H}}(err_{\mathcal{S}}(h)+d_{h,\mathcal{H}}(\mathcal{S},\mathcal{T})-err_{\widehat{\mathcal{S}}}(h)-d_{h,\mathcal{H}}(\widehat{\mathcal{S}},\widehat{\mathcal{T}}))\\
& = \sup_{h\in\mathcal{H}}(err_{\mathcal{S}}(h)-err_{\widehat{\mathcal{S}}}(h)+d_{h,\mathcal{H}}(\mathcal{S},\mathcal{T})-d_{h,\mathcal{H}}(\widehat{\mathcal{S}},\widehat{\mathcal{T}}))\\
& \leq \sup_{h\in\mathcal{H}}(err_{\mathcal{S}}(h)-err_{\widehat{\mathcal{S}}}(h))+\sup_{h\in\mathcal{H}}(d_{h,\mathcal{H}}(\mathcal{S},\mathcal{T})-d_{h,\mathcal{H}} (\widehat{\mathcal{S}},\widehat{\mathcal{T}})) \,. \\
\end{split}
\end{equation}
\end{small}%
First by Lemma \ref{rademacher bound}, $\forall \delta \ge 0$, with probability $1-\delta$,
\begin{small}
\begin{equation}\nonumber
\begin{split}
\sup_{h\in\mathcal{H}}(err_{\mathcal{S}}(h)-err_{\widehat{\mathcal{S}}}(h))\leq 2\mathfrak{R}_{n,\mathcal{S}}(\mathcal{H})+\sqrt{\frac{log\frac{2}{\delta}}{2n}} \,.
\end{split}
\end{equation}
\end{small}%
Then the difference between $d_{h,\mathcal{H}}(\mathcal{S},\mathcal{T})$ and $d_{h,\mathcal{H}}(\widehat{\mathcal{S}},\widehat{\mathcal{T}})$:
\begin{small}
\begin{equation}\nonumber
\begin{split}
& d_{h,\mathcal{H}}(\mathcal{S},\mathcal{T})-d_{h,\mathcal{H}}(\widehat{\mathcal{S}},\widehat{\mathcal{T}})\\
& = \sup_{h' \in \mathcal{H}}(dis_\mathcal{T}(h',h)-dis_{\mathcal{S}}(h',h))-\sup_{h' \in \mathcal{H}}(dis_{\widehat{\mathcal{T}}}(h',h)-dis_{\widehat{\mathcal{S}}}(h',h))\\
& \leq \sup_{h' \in \mathcal{H}}(dis_\mathcal{T}(h',h)-dis_{\mathcal{S}}(h',h)-dis_{\widehat{\mathcal{T}}}(h',h)+dis_{\widehat{\mathcal{S}}}(h',h))\\
& \leq \sup_{h' \in \mathcal{H}}(dis_\mathcal{T}(h',h)-dis_{\widehat{\mathcal{T}}}(h',h))-\sup_{h' \in \mathcal{H}}(dis_{\mathcal{S}}(h',h)-dis_{\widehat{\mathcal{S}}}(h',h)) \,. \\
\end{split}
\end{equation}
\end{small}%
Take supremum over $h\in\mathcal{H}$, we have:
\begin{small}
\begin{equation}\nonumber
\begin{split}
&\sup_{h\in\mathcal{H}}(d_{h,\mathcal{H}}(\mathcal{S},\mathcal{T})-d_{h,\mathcal{H}}(\widehat{\mathcal{S}},\widehat{\mathcal{T}}))\\
& \leq  \sup_{h' \in \mathcal{H}}|dis_\mathcal{T}(h',h)-dis_{\widehat{\mathcal{T}}}(h',h)|-\sup_{h' \in \mathcal{H}}|dis_{\mathcal{S}}(h',h)-dis_{\widehat{\mathcal{S}}}(h',h)| \,.
\end{split}
\end{equation}
\end{small}%
From Proposition \ref{DD rademacher bound}, we can directly get:
\begin{small}
\begin{equation}\nonumber
\begin{split}
&\sup_{h\in\mathcal{H}}(d_{h,\mathcal{H}}(\mathcal{S},\mathcal{T})-d_{h,\mathcal{H}}(\widehat{\mathcal{S}},\widehat{\mathcal{T}})) \\
&\leq  2\mathfrak{R}_{n,\mathcal{S}}(\mathcal{G}_\Gamma\mathcal{H})+\sqrt{\frac{log\frac{2}{\delta}}{2n}}+2\mathfrak{R}_{m,\mathcal{T}}(\mathcal{G}_\Gamma\mathcal{H})+\sqrt{\frac{log\frac{2}{\delta}}{2m}}.
\end{split}
\end{equation}
\end{small}%
Combine the above two parts of inequality, we can get the final generalization bound.
\end{proof}

\end{document}